\documentclass{article}

\usepackage{microtype}
\usepackage{graphicx}
\usepackage{subfigure}
\usepackage{booktabs} 

\usepackage{hyperref}

\usepackage[accepted]{icml2024}

\usepackage{amsmath}
\usepackage{amssymb}
\usepackage{mathtools}
\usepackage{amsthm}
\usepackage{multicol}

\usepackage[capitalize,noabbrev]{cleveref}

\theoremstyle{plain}
\newtheorem{theorem}{Theorem}[section]

\newtheorem{proposition}[theorem]{Proposition}
\newtheorem{lemma}[theorem]{Lemma}

\theoremstyle{plain}
\newtheorem{definition}[theorem]{Definition}

\theoremstyle{plain}
\newtheorem{example}[theorem]{Example}

\theoremstyle{plain}
\newtheorem{remark}[theorem]{Remark}

\newcommand{\shnote}[1]%
    {\textcolor{blue}{ #1}}
\newcommand{\diag}			{{\text{diag}}}

\usepackage[textsize=tiny]{todonotes}

\icmltitlerunning{Mollification Effects of Policy Gradient Methods}

\begin{document}

\twocolumn[
\icmltitle{Mollification Effects of Policy Gradient Methods}




\begin{icmlauthorlist}
\icmlauthor{Tao Wang}{yyy}
\icmlauthor{Sylvia Herbert}{yyy}
\icmlauthor{Sicun Gao}{yyy}
\end{icmlauthorlist}

\icmlaffiliation{yyy}{University of California, San Diego, La Jolla, USA}

\icmlcorrespondingauthor{Tao Wang}{taw003@ucsd.edu}

\icmlkeywords{Machine Learning, ICML}

\vskip 0.3in
]



\printAffiliationsAndNotice{}  

\begin{abstract}
Policy gradient methods have enabled deep reinforcement learning (RL) to approach challenging continuous control problems, even when the underlying systems involve highly nonlinear dynamics that generate complex non-smooth optimization landscapes. We develop a rigorous framework for understanding how policy gradient methods mollify non-smooth optimization landscapes to enable effective policy search, as well as the downside of it: while making the objective function smoother and easier to optimize, the stochastic objective deviates further from the original problem. We demonstrate the equivalence between policy gradient methods and solving backward heat equations. Following the ill-posedness of backward heat equations from PDE theory, we present a fundamental challenge to the use of policy gradient under stochasticity. Moreover, we make the connection between this limitation and the uncertainty principle in harmonic analysis to understand the effects of exploration with stochastic policies in RL. We also provide experimental results to illustrate both the positive and negative aspects of mollification effects in practice.
\end{abstract}

\section{Introduction}


Deep reinforcement learning (RL), especially methods based on policy gradients~\cite{silver, lillicrap, schulman17}, has been successfully used to solve challenging nonlinear control problems~\cite{gu16, Kim-RSS-22}. The approach considers control design as an optimization problem over the parameter space of the policy, and thus the effectiveness of policy gradient methods depend heavily on the {\em optimization landscape} in the policy space. 
Although the global convergence of policy gradient methods has been established for restricted classes of systems such as linear dynamics or tabular state spaces~\cite{fazel, agarwal}, its observed effectiveness for nonlinear systems is not well understood. In fact, it has been shown that in many control settings, especially those with chaotic dynamics, the RL formulation can generate highly non-smooth optimization landscapes \citep{suh, twang} that should challenge the use of gradient-based optimization methods including policy gradient. Existing attempts on bridging this gap have mostly focused on the effectiveness of exploration~\cite{haarnoja18, cai20d}. However, while exploration is a common element in every search-based algorithm, it alone does not fully explain the success of RL over other schemes in high-dimensional spaces.

\begin{figure}[t!]
\centering
\includegraphics[width=0.45\textwidth]{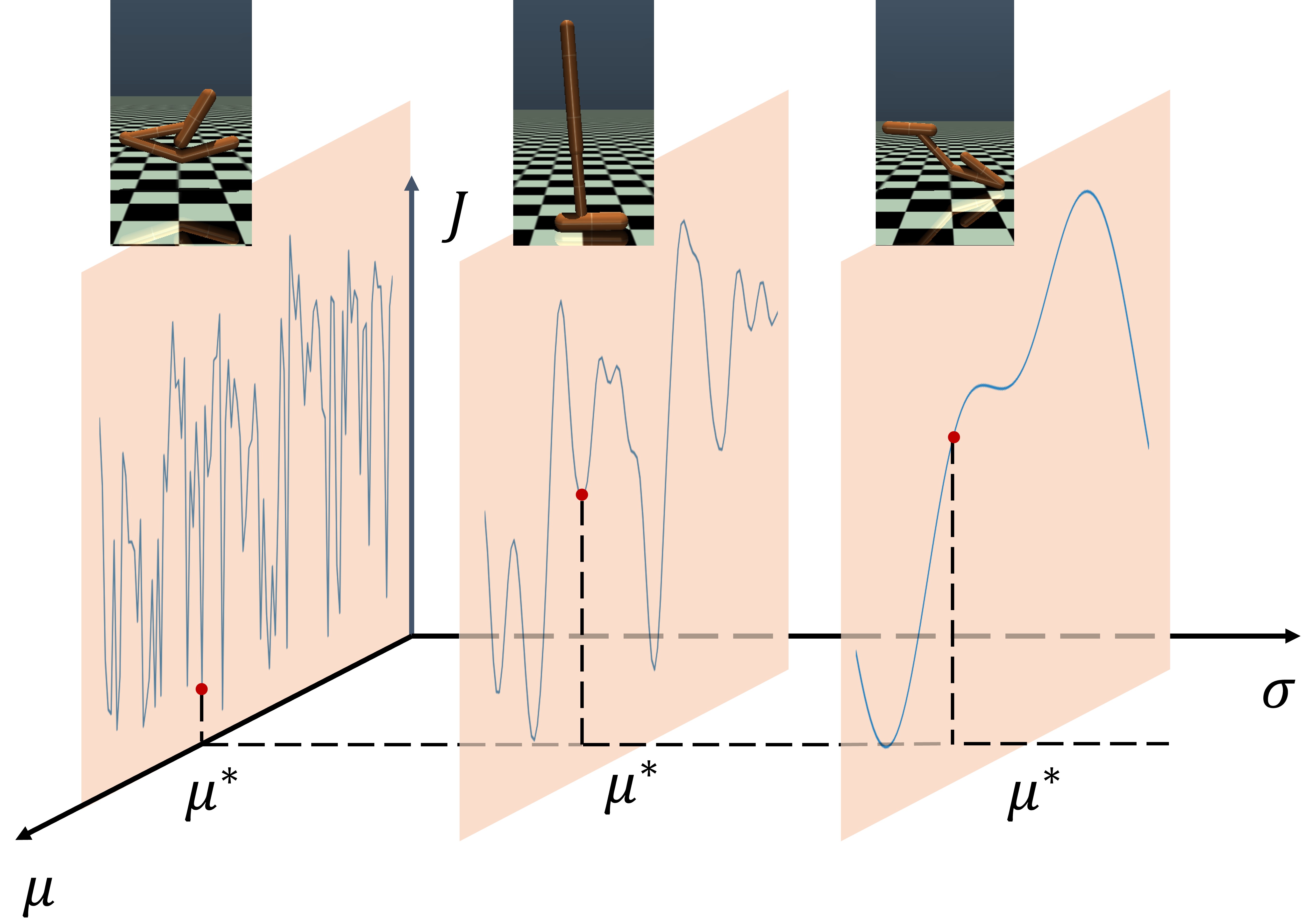}
\caption{The Gaussian kernel in the policy gradient mollifies the optimization landscape. However, when the variance $\sigma^2$ is too small, the landscape remains highly non-smooth. Conversely, if the variance is too large, the Gaussian kernel over-smooths the landscape, eliminating the optimal solution. Both of these lead to failures in the hopper stand task. Details are avaliable in Section~\ref{sec:experiment}.}
\label{fig:illustration}

\end{figure}

Given that the true gradient of the objective function may not exist in many robotics systems, we aim to understand the effectiveness of policy gradient methods through the analytic perspectives of partial differential equations (PDEs) and stochastic dynamical systems. We view the Gaussian noise introduced in the stochastic policies as a smoothing kernel that {\em mollifies} the objective function. Mollification is a concept in analysis that corresponds to smoothing the sharp features of a non-smooth function while remaining close to the original one. In effect, regardless of whether the original landscape is smooth or not, policy gradient methods can be consistently applied to estimate the gradient of the {\em mollified objective}. We draw on PDE theory and establish the equivalence between policy gradient methods and the Cauchy problem for heat equations. In particular, we show that training a policy by stochastic policy gradient algorithms is equivalent to performing gradient ascent for the solution of the corresponding heat equation in the spatio-temporal domain. 
In control problems where deterministic control policies are expected as the outcome, the RL learning process 
corresponds to a time-reversed heat process. However, the backward Cauchy problem for heat equations is ill-posed in terms of the stability of the solution, which becomes less smooth as time decreases \citep{hollig, kabanikhin}; this suggests that reducing the variance of stochastic policy can result in a more non-smooth objective. This effect is especially pronounced when the MDP is chaotic, causing the true value function to contain significant high-frequency components. 

Importantly, the results illustrate a fundamental limitation of policy gradient methods. While the mollification smooths the optimization landscape, it unavoidably changes the original objective. 
This trade-off can be precisely formulated by the uncertainty principle from harmonic analysis, as we visualize in Figure~\ref{fig:illustration}. That is, when the variance of the stochastic policy is too large, the mollified function will deviate too much from the true objective, which makes the estimated gradient no longer informative. 
On the other hand, if the variance is too small, the mollification effect is weak and hence leads to another highly-oscillating landscape. 
In either case, the training process becomes unstable, suggesting the existence of an optimal variance for the stochastic policy in policy gradient methods. 

Equipped with the theoretical results, we conduct experiments to illustrate how our framework can be applied to explain both the successes and failures in practice. In particular, from the view of mollification, we can characterize a class of control problems where RL algorithms consistently face challenges: the region of attraction for the optimal policy is extremely small and thus can be entirely eliminated by the Gaussian kernel in stochastic policies. It also explains why policy gradient methods always encounter difficulties when solving quadrotor-related problems and a detailed discussion is presented in Section~\ref{sec:experiment}.

The contributions of this paper are summarized as follows:

\noindent $\bullet$ We establish a framework that builds the connection between policy gradient methods and heat equations to study their mollification effects. It provides an explanation why policy gradient can still yield effective ascent directions even when the optimization landscape is fractal.

\noindent $\bullet$  We analyze the fundamental trade-off of policy gradient from the perspective of harmonic analysis. In particular, we demonstrate that the variance of stochastic policies should be carefully balanced. The training process loses stability if it is either too small or too large.

\noindent $\bullet$  Numerical results are presented to substantiate our theoretical results. We show that policy gradient methods exhibit limited effectiveness in specific MDPs, such as the quadrotor system. In these cases, the policy landscape has a spike-like structure around the optimal policy, which is then filtered by the mollification effect of policy gradient.

\section{Related Work} 



\paragraph{Optimization Landscapes in RL.}

When the state space in an RL problem is assumed to be finite, the corresponding policy landscape is smooth (though may be non-convex). This allows for various gradient-based algorithms to converge towards the optimal policy \cite{agarwal, russo, xiao, zeng}. Similar results can be obtained for some simple continuous state-space MDPs, such as in linear-quadratic regulator (LQR) \cite{fazel}, linear-quadratic-Gaussian control (LQG) \cite{tang} and robust control \cite{zhang}. For general control settings, however, even the smoothness of objective function is not guaranteed. In particular, it has been shown that the variance of the gradient estimator in chaotic systems is very likely to explode \cite{parmas, metz} due to the long chain of non-linear computations. This phenomenon is a reflection of the fractal structure in both value and policy landscapes \cite{twang}. There has been work on mitigating the effect of non-smooth landscapes, such as local smoothing via Gaussian kernels that act as low-pass filters to block high-frequency components \cite{parmas, suh, zhang23s}, as well as reparameterization techniques \cite{parmas18, parmas21}. Our work contributes to this line by providing a theoretical analysis of the strengths and limitations of the smoothing effect produced by policy gradient methods in the case of fractal landscapes.

\begin{figure}[h!]
\centering

\subfigure[]{\includegraphics[width=0.23\textwidth]{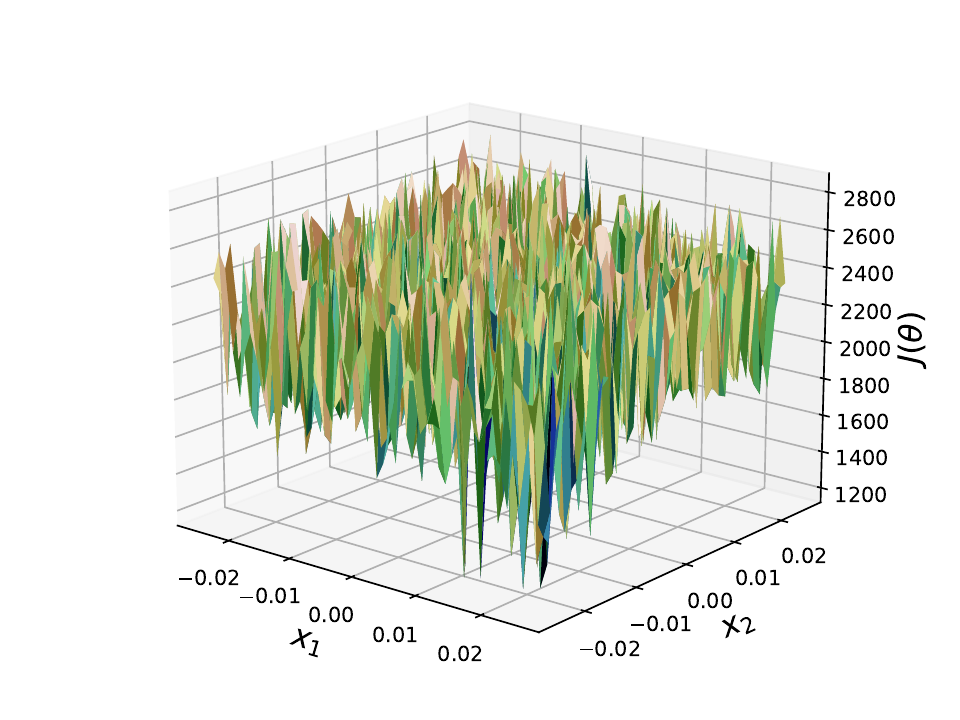}}
\subfigure[]{\includegraphics[width=0.23\textwidth]{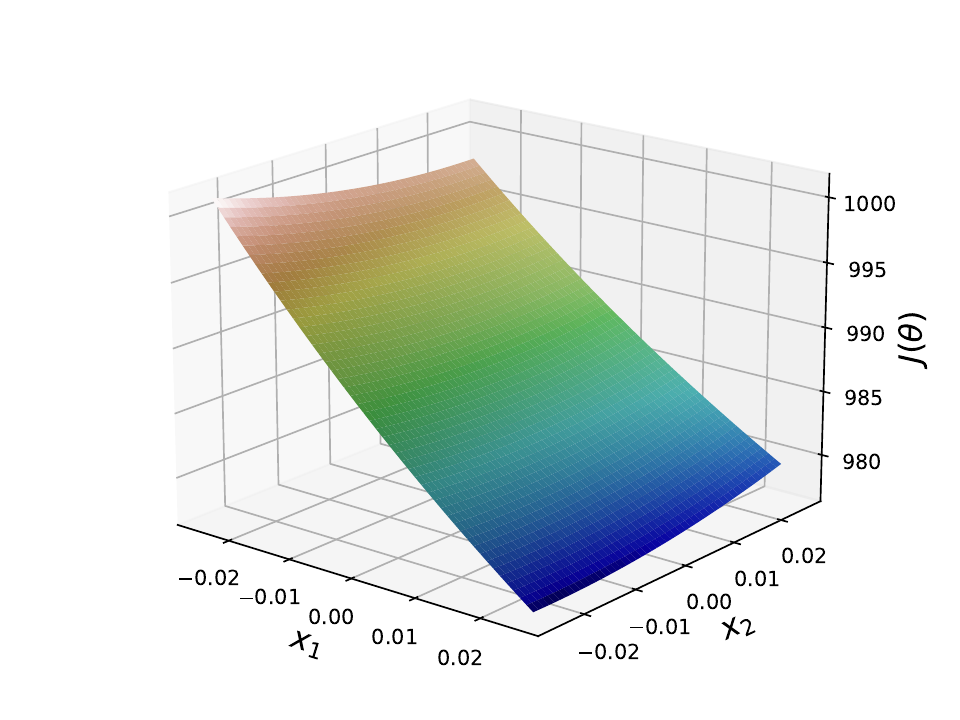}}

\caption{Fractal landscapes occur in chaotic MDPs. For instance, the objective landscape of the double pendulum system as shown in (a) has a fractal structure, in contrast to the non-chaotic single pendulum system in (b). Both systems are controlled by deterministic neural network policies.}
\label{fig:heat}


\end{figure}

\paragraph{Policy Gradient over Non-Smooth Landscapes.} It is observed that policy gradient methods can still provide good solutions when the optimization landscape is non-smooth or fractal \citep{lillicrap, schulman15}. A widely-accepted explanation for its effectiveness is that it encourages the policy to keep exploring undiscovered, high-reward actions \cite{schulman17, cai20d} that are out of the current policy distribution, thereby increasing the likelihood of discovering better policies. Various exploration schemes can be found in actor-critic approaches \cite{haarnoja18}, Q-learning methods \cite{sutton98, jin} and traditional control algorithms \cite{tedrake09}. 

Most deep RL algorithms including policy gradient can be considered as generalized policy iteration methods~\cite{william, sutton99, kakade, silver, schulman15, schulman17}. Typically, they consist of two interacting stages: value approximation and policy improvement. Q-learning is one of the popular methods in model-free settings \cite{watkins, hasselt, hester}, aiming to estimate the Q-function to guide future actions. Actor-critic is another framework that interactively updates the value function and current policy \cite{konda, peters, fujimoto, haarnoja}. There are also works studying value approximation from both theoretical and practical sides \cite{pmlr-v37-schaul15, yiningwang, jin20a}. In this paper, we will focus on the policy improvement stage and presume that the value function is exact.

\paragraph{Mollification in Stochastic Optimization.} Policy gradient is not the only domain where stochasticity demonstrates its strength in mollifying the potentially non-smooth objective functions. 
Algorithms in stochastic optimization, as well as those in zeroth-order optimization, benefit from the mollification effect. \cite{wierstra, shamir, wang18e, bohm}. Their ability to generalize the non-smoothness in objective functions, which is intractable for classical methods, motivates a line of applications in machine learning \cite{agarwal2, lin, chen23}.

\section{Preliminaries}
\label{sec:background}

We consider infinite-horizon MDPs with state space $\mathcal{S} \subseteq \mathbb{R}^n$ and action space $\mathcal{A} = \mathbb{R}^m$. 
The initial state $s_0$ is sampled from distribution $\rho_0$, and at each step $k\geq 0$, the action $a_k$ at state $s_k$ is obtained from the (stochastic or deterministic) policy $\pi_{\theta}(\cdot | s_k)$, parameterized by $\theta \in \mathbb{R}^N$. 
The objective of the RL problem is to maximize the performance metric: 
\begin{eqnarray}   \label{loss}
    J(\theta) &=& \mathbb{E}_{a_k \sim \pi_{\theta}(\cdot | s_k), s_0 \sim \rho_0} [ \sum_{k = 0}^{\infty} \gamma^{ k}  R(s_k, a_k) ]\\
    &=& \int_{\mathcal{S}} \rho^{\pi_{\theta}}(s) \int_{\mathcal{A}} Q^{\pi_{\theta}}(s, a) \pi_\theta (a|s) \ \mathrm{d}a \mathrm{d}s
\end{eqnarray}
where $\gamma \in (0, 1)$ is the discount factor and $R(s, a)$ is the reward function. 
We will assume that the state space $\mathcal{S}$ is compact and the objective function $J(\theta)$ is bounded. In the integral form above, $\rho^\pi(\cdot)$ is the discounted visitation density under $\pi$ (assuming it exists) and $Q^{\pi}$ is the $Q$-function of $\pi$.

\paragraph{Policy gradient methods. } The policy gradient theorem  provides an explicit form of the gradient of the performance objective over the policy parameters that can be estimated by state-action samples~\cite{sutton99}:
\begin{equation}
\label{pgth}
    \nabla_\theta J(\theta) \propto \int_{\mathcal{S}} \rho^{\pi}(s) \int_{\mathcal{A}} Q^{\pi}(s, a) \nabla_\theta \pi_\theta (a|s) \ \mathrm{d}a \mathrm{d}s,
\end{equation}
Dropping $\theta$ from $\pi_{\theta}$ in $\rho^{\pi}$ and $Q^{\pi}$ here is intentional, to highlight the core of the theorem in showing that the gradient operator $\nabla_{\theta}$ only needs to be applied to the $\pi_{\theta}(a|s)$ part of the integrand~\cite{sutton99}. 
It follows that $\nabla_{\theta} J(\theta)$ can be estimated by 
\begin{equation}
\label{eq:gradest}
\nabla_{\theta} J(\theta) \simeq \hat{\mathbb{E}}_{s,a\sim \pi_{\theta}}[A^{\pi_{\theta}}(s, a) \nabla_\theta \pi_\theta (a|s)]
\end{equation}
where $A^{\pi_{\theta}}(s, a) = Q^{\pi_{\theta}}(s, a) - V^{\pi_{\theta}}(s)$ is the advantage function and $V^{\pi_{\theta}}$ is the value function of $\pi_{\theta}$~\cite{advantage}, and $\hat{\mathbb{E}}$ denotes sampled mean, which converges to the integral of (\ref{pgth}) under some necessary smoothness assumptions. 

\paragraph{Fractal landscapes in RL.}

One fundamental assumption for policy gradient methods to work is that the objective function $J(\theta)$ \emph{does have a gradient}, so that \eqref{eq:gradest} can provide a close estimate of it. This assumption, however, may not hold when the underlying MDP is chaotic and has a positive maximal Lyapunov exponent (MLE). Indeed, we have the following result that characterizes the fractal structure in the policy space:
\begin{proposition}
    \citep{twang} Assume that the dynamics, reward function and policy are all Lipschitz continuous with respect to their input variables. Let $\pi_\theta$ be a deterministic policy and $\lambda(\theta)$ denote the MLE of the system. Suppose that $\lambda(\theta) > -\log \gamma$, then
    \begin{itemize}
        \item $V^{\pi_\theta}(\cdot)$ is $\frac{-\log \gamma}{\lambda(\theta)}$-H\"older continuous;

        \item $J(\cdot)$ is $\frac{-\log \gamma}{\lambda(\theta)}$-H\"older continuous.
    \end{itemize}
\end{proposition}
Note that a function $f(x)$ is $\alpha$-H\"older continuous at $x$ if $|f(x) - f(x')| \leq K |x - x'|^{\alpha}$ for some constant $K > 0$ when $|x - x'|$ is small, which typically indicates non-smoothness when $\alpha < 1$. It reduces to Lipschitz continuity when $\alpha = 1$. In general, the true gradient $\nabla J(\theta)$ may not exist if the MLE is positive. This prompts us to understand the mechanism of \eqref{eq:gradest} from the perspective of mollification, rather than focusing solely on gradient estimation.

\paragraph{Cauchy problem for heat equations.} 

The heat equation, as known as the diffusion equation, describes the evolution of the distribution of temperature in the space $\mathbb{R}^m$ \cite{evans}:
\begin{equation}  \label{heat}
    \begin{cases}
      & 2 u_t - \Delta u = 0, \quad (x, t) \in \mathbb{R}^m \times (0, \infty)\\
      & u = g, \quad (x, t) \in \mathbb{R}^m \times \{ 0 \}.
    \end{cases} \\
\end{equation}
where $x \in \mathbb{R}^m$ is the position and $t \in [0, \infty)$ is the time, $u(x, t)$ is the temperature at $(x, t) \in \mathbb{R}^m$, $\Delta u = u_{x_1 x_1} + ... + u_{x_m x_m}$ is the Laplacian of $u$ and $g$ is the initial distribution of temperature at $t = 0$. The solution to this PDE system is given by 
\begin{equation}
\label{eq:heatsol}
    u(x, t) = \int_{\mathbb{R}^m} g(z) \Phi(x - z, t) \ \mathrm{d}z,
\end{equation}
where $\Phi(x - z, t) = \frac{1}{(2 \pi t)^{m/2}} e^{-\| x - z \|^2 / 2 t}$ is the heat kernel. The solution $u(x, t)$ becomes smooth as $t$ increases, {\em regardless of} how non-smooth the original function $g$ is. This is due to the mollification effect of the heat kernel $\Phi(x-z, t)$, as shown in Figure~\ref{fig:heat} (a).

\section{The Dynamics of Policy Improvement}
\label{sec:heat}


In this section, we will establish the connection between policy gradient methods and the heat equation. We will focus on analyzing policies that are parameterized as isotropic Gaussian distributions, i.e. $\pi_\theta(a|s)=\mathcal{N}(\mu(s),\sigma^2\mathcal{I}_m)$ for fixed state $s \in \mathcal{S}$ , where $\mu(s) \in \mathbb{R}^m$ and $\sigma^2\mathcal{I}_m$ is the covariance matrix with nonzero scalar $\sigma > 0$ and $\mathcal{I}_m$ being the $m\times m$ identity matrix. 
Consider the inner integral of \eqref{pgth}, namely 
\begin{equation}
\label{eq:sur}
     L_s(\theta) :=  \int_{\mathcal{A}} Q^{\pi}(s, a)  \pi_\theta (a|s) \ \mathrm{d}a,
\end{equation}
which is a \emph{mollification} of $Q^{\pi}$ in the action space and hence is smooth with respect to $\theta$ even in the case of chaotic MDPs. This is because the policy density function $\pi_\theta (a|s)$ works as a smoothing kernel that mollifies the (often non-smooth) landscape of $Q^\pi(s, a)$ in $\mathcal{A}$. Since \eqref{eq:sur} is where mollification occurs, we mainly focus on $L_s(\theta)$.

\begin{figure}[h!]
\centering

\subfigure[]{\includegraphics[width=0.23\textwidth]{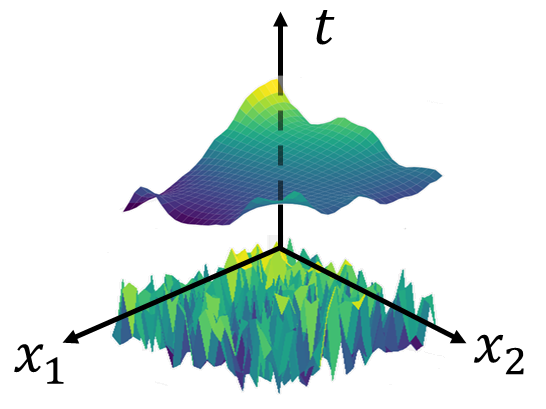}}
\subfigure[]{\includegraphics[width=0.24\textwidth]{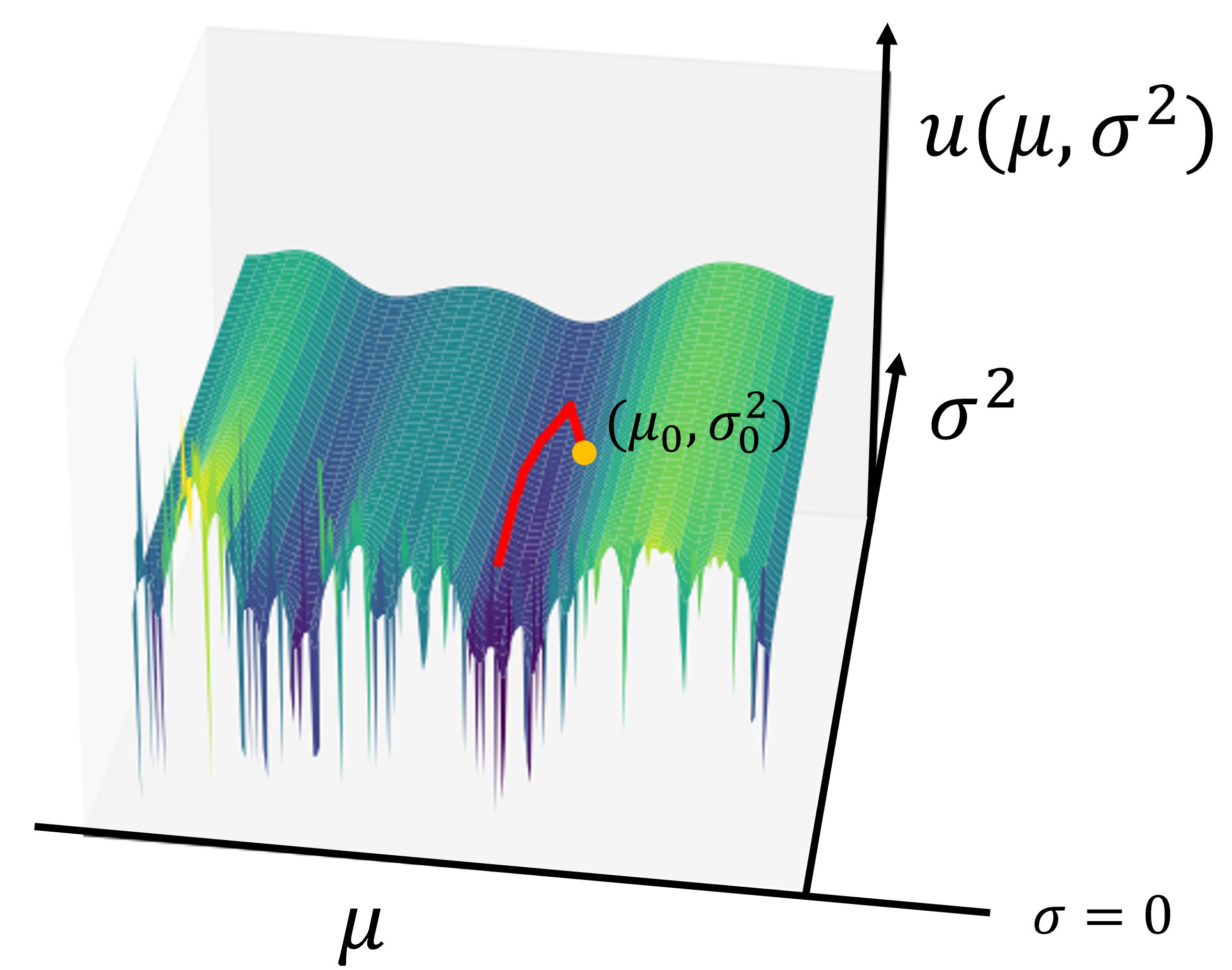}}

\caption{(a) The heat equation smooths the initial temperature distribution as $t$ increases; (b) The gradient flow of $u(\mu, \sigma^2)$ in the solution space.}
\label{fig:heat}

\end{figure}



To see its connection with the Cauchy problem for heat equations, let us first ignore the policy parameterization and only take the mean and variance of $\pi_\theta$ into account, i.e., assume $L_s(\theta) = L_s(\mu, \sigma^2)$ for now, which further has 
$$L_s(\mu, \sigma^2) = \int_{\mathbb{R}^m} Q^{\pi}(s, z) \Phi(\mu - z, \sigma^2) \ \mathrm{d}z.$$
Thus, if we consider $\mu$ and $\sigma^2$ as position and time, respectively, then $L_s(\mu, \sigma^2)$ is exactly the solution of \eqref{heat}. Indeed, we have the formal connection between them:
\begin{proposition}
\label{th:equivalence}
    Let $L_s(\mu, \sigma^2)$ is given by \eqref{eq:sur} where $\pi_\theta \sim \mathcal{N}(\mu,\sigma^2\mathcal{I}_m)$, then $L_s(\mu, \sigma^2)$ is the solution of the heat equation \eqref{heat}.
\end{proposition}
This connection enables us to analyze the dynamics of policy gradient methods from the view of PDE theory.

\paragraph{Smoothing by mollication.}


Essentially, the heat equation mollifies the solution through the Gaussian kernel, which acts as a low-pass filter and blocks high-frequency components. Consider the one-dimensional case in which we expand the initial condition $g(x) = \sum_{k = -\infty}^\infty B_k  e^{i k x}$ in Fourier series (here we assume its existence for the purpose of illustration), the solution of \eqref{heat} is given by
$$\tilde{u}(x, t) = \sum_{k = -\infty}^\infty B_k e^{-k^2 t} e^{i k x}$$
where $t \geq 0$. Therefore, for any frequency $k \in \mathbb{Z}$, the magnitude of this frequency $|B_k e^{-k^2 t}|$ decays exponentially fast as $t$ increases, especially for higher frequencies (larger $k$). It means that the heat equation mollifies functions by suppressing their high-frequency components that cause fractal structures in the optimization landscape.




\paragraph{Mollified optimization landscape.}

Understanding that policy gradient methods are equivalent to performing gradient descent on the solution $u(\mu, \sigma^2)$ in \eqref{heat} over the $(\mu, \sigma^2)$-space allows us to better predict the behavior of policy gradient in the solution space of the heat equation (Figure~\ref{fig:heat}(b)). In particular, an important property of the heat equation is the strong maximum principle: 
\begin{proposition}
[Strong Maximum Principle \cite{evans}]
\label{prop:heat}
    Let $u$ be the solution of \eqref{heat}. Suppose that $(x_0, t_0) \in \mathbb{R}^d \times (0, \infty)$, then for any $\delta_1 > 0$ and $\delta_2 \in (0, t_0)$, exactly one of the following statements is true:
    
\noindent (I) There exists $(x', t') \in \mathcal{B}(x_0, \delta_1) \times (t_0 - \delta_2, t_0)$ such that $u(x', t') > u(x, t)$;
 
\noindent (II) $u$ is constant in $\overline{\mathcal{B}(x_0, \delta_1)} \times [0, t_0]$. 
\end{proposition}


This results states that for any $(x_0, t_0)$, it cannot be a strict local maximum of $u(x, t)$ when $t_0 > 0$. Following the equivalence shown in Proposition~\ref{th:equivalence}, we see that for any mean-variance pair $(\mu, \sigma^2)$ with some positive $t$, 
there always exists at least one increasing direction in a neighborhood of $(\mu, \sigma^2)$. It also explains why policy gradient algorithms can improve the performance even in the case that the true optimization landscape is fractal. The convergence modes of policy gradient follow from the connection:
\begin{theorem}
\label{th:flow1}
    Let $\pi_\theta$ be an isotropic Gaussian policy $\mathcal{N}(\mu,\sigma^2\mathcal{I}_m)$. 
    Suppose that $\nabla_{(\mu, \sigma^2)} L_s(\mu, \sigma^2) = 0$, then either $\pi_\theta$ is deterministic, or $\pi_\theta$ is stochastic and $\theta$ is not a strict local maximum of $L_s(\theta)$. Then exactly one of the following statements holds:
    
    \noindent(I) $\pi_\theta$ is deterministic; 
    
    \noindent(II) $\pi_\theta$ is stochastic and $\theta$ is not a strict local maximum of $L_s(\theta)$.      
\end{theorem}
\begin{proof}
    Suppose that the policy is stochastic, i.e., $\sigma > 0$. According to Proposition~\ref{prop:heat}, for any $\delta_1, \delta_2 > 0$, $L_s(\mu, \sigma^2)$ is either constant in $\overline{\mathcal{B}(\mu_0, \delta_1)} \times [0, \sigma^2_0]$, or there exists another $[\mu', \sigma'^2] \in \mathcal{B}(\mu, \delta_1) \times (\sigma^2 - \delta_2, \sigma^2)$, which implies that $\pi_\theta$ cannot be a strict local maximum of $L_s(\theta)$.
\end{proof} 

Similar results can be obtained for the gradient-descent case using the Strong Minimum Principle. In other words, if policy gradient converges toward a strict local maximum of $L_s(\mu, \sigma^2)$, we know that the final policy must be deterministic. It confirms the appropriate use of policy gradient methods in continuous control problems, where the control policy used in practice should be deterministic.

\paragraph{Policy parameterization.}
In policy gradient methods, the mean of the Gaussian policy is typically parameterized through some differentiable representation, namely $\mu = \mu(s_0; \zeta)$ where $\zeta \in \mathbb{R}^{N_0}$ denotes the parameters in the representation, such as weights and biases in a neural network. 
In this case, we have $\theta = [\zeta, \sigma^2]$ as the policy parameters. Note that 
\[\nabla_\zeta L_s(\zeta, \sigma^2) =\nabla_\mu u(\mu(s_0; \zeta) \ \frac{\partial \mu}{\partial \zeta} \Big|_{s = s_0},\]
where $\frac{\partial \mu}{\partial \zeta}$ is the Jacobian matrix of $\mu$ with respect to the parameter $\zeta$, and the full gradient is given as
\begin{equation}
\label{eq:parameterized}
    \nabla_\theta L_s(\theta) = [\nabla_\mu u(\mu(\zeta), t) \ \frac{\partial \mu}{\partial \zeta} \Big|_{s = s_0}, u_{\sigma^2}(\mu(\zeta), \sigma^2))]^T.
\end{equation}
Therefore, the degeneracy of the parameterization $\mu(s_0; \zeta)$ plays a crucial role in analyzing the gradient flow of \eqref{eq:parameterized}: given an arbitrary initial state $s_0$, the Jacobian $\frac{\partial \mu(s; \zeta)}{\partial \zeta} \Big|_{s = s_0}$ is a linear transformation from the action space $\mathcal{A} = \mathbb{R}^m$ to the representation parameter space $\mathbb{R}^{N_0}$. This mapping is non-degenerate if it is injective, i.e., $\ker \ \frac{\partial \mu}{\partial \zeta} \Big|_{s = s_0} = \{ \textbf{0} \}$, which establishes a one-to-one correspondence between the gradient flow in Theorem~\ref{th:flow1} and in the new parameterization space $\mathbb{R}^{N_0}$: 

\begin{theorem}
\label{th:parameterized}
    Let $\pi_\theta(\cdot|s_0) \sim \mathcal{N}(\mu, \sigma^2 \mathcal{I}_d)$ be an isotropic Gaussian policy where $\mu = \mu(s_0; \zeta)$ is parameterized by $\zeta \in \mathbb{R}^{N_0}$. Also, assume that $\frac{\partial \mu}{\partial \zeta} \Big|_{s = s_0}$ is injective. Suppose that $\nabla_\theta L_{s_0}(\theta) = 0$ where $\theta = [\zeta, \sigma^2]^T$, then exactly one of the following statements holds:

\noindent (I) $\pi_\theta$ is deterministic; 
    
\noindent (II) $\pi_\theta$ is stochastic and $\theta$ is not a strict local maximum of $L_{s_0}(\theta)$.

\end{theorem}

This result also encourages the use of complex representations, such as neural networks, since the more parameters they contain, the more likely the Jacobian $\frac{\partial \mu(x; \zeta)}{\partial \zeta} \Big|_{x = s}$ is to be non-degenerate for arbitrary $s_0 \in \mathcal{S}$. For the whole gradient estimator $\nabla_\theta \hat{J}$, the component $\nabla_\zeta \hat{J}$ is given by
\begin{align*}
    \nabla_\zeta \hat{J} &= \int_{\mathcal{S}} \rho^{\pi}(s) \nabla_\zeta L_s(\theta) \ \mathrm{d}s  \\
    &= \int_{\mathcal{S}} \rho^{\pi}(s)  \nabla_\mu u(\mu(s; \zeta), t) \ \frac{\partial \mu(x; \zeta)}{\partial \zeta} \Big|_{x = s} \ \mathrm{d}s,
\end{align*}
where $\rho^\pi(s)$ acts as a weighting function on $\nabla_\theta L_s(\theta)$. It should note that the results in Theorem~\ref{th:parameterized} no longer hold true, as the gradient directions at different states may cancel out. Therefore, while we can still have a gradient estimate even if the true policy objective is non-differentiable, there is no guarantee that policy gradient algorithms will converge towards some good solutions in general.

\paragraph{Anisotropic Gaussian distributions.}

Now let us consider the case of diagonal covariance matrices $\Sigma = \diag(r_1, ..., r_m)$ with $r_1, ..., r_m > 0$, which is a common practice in most policy gradient methods \cite{schulman15, schulman17}. Similar to heat equations, we will show that the strong maximum principle still holds for \eqref{eq:aniso}. Let $t \geq 0$ and consider the following parabolic equation
\begin{equation}
\label{eq:aniso}
     u_t - \frac{1}{2} \sum_{i=1}^{m} r_i \ u_{x_i x_i} = 0,  
\end{equation}
which has $u(\mu, t) = L_s(\mu, t \mathbf{r})$, where $\mathbf{r} = [r_1, ..., r_m]^T$. According to the maximum principle of parabolic equations \cite{evans}, the same claims as in Proposition~\ref{prop:heat} hold for \eqref{eq:aniso}. Since any small neighborhood of the point $(\mu, 1)$ can always be embedded into a neighborhood of $(\mu, \mathbf{r})$, we have proved the following result:
\begin{theorem}
\label{th:aniso}
    Let $\pi_\theta$ be a Gaussian policy $\mathcal{N}(\mu,\diag(\mathbf{r}))$ with diagonal covariance matrix where $\mathbf{r} = [r_1, ..., r_m]^T$ is independent of states. 
    Suppose that $\nabla_{(\mu, \mathbf{r})} L_s(\mu, v) = 0$, then exactly one of the following statements holds:
    
    \noindent(I) $r_i = 0$ for some $i = 1, ..., m$; 
    
    \noindent(II) $r_i$ are all positive and $\theta$ is not a strict local maximum of $L_s(\theta)$. 
\end{theorem}


\section{The Limitations of Mollification}
\label{sec:uncertainty}

In the previous section, we have shown that policy gradient methods are closely related to the heat equation whose solution mollifies the initial condition and hence enables gradient-based algorithms. In this section, we will move to analyze the downsides of mollification effects and describe the fundamental trade-off of policy gradient methods between smoothing and approximating.




\subsection{Convergence to deterministic policies}

In control tasks, the policy gradient is expected to converge towards some deterministic policy, which prompts us to ask whether there exists a smooth gradient flow in a neighborhood of $\mathbb{R}^m \times \{ 0 \}$ in the $(\mu, \sigma^2)$-space. Unfortunately, the answer is \textit{NO} in the general case: as the heat equation makes the solution smoother and smoother as $t$ increases, moving backwards in time will make $u(x, t)$ less smooth, as illustrated in Figure~\ref{fig:heat} (a). In other words, the gradient flow can be smoothly extended to $\mathbb{R}^m \times \{ 0 \}$ only if $\lim_{t \rightarrow 0^+} \int_{\mathbb{R}^m} g(z) \frac{\partial \Phi(x - z, t)}{\partial t} \ dz$ exists for every $x \in \mathbb{R}^m$, which is true for smooth functions and is discussed in Appendix~\ref{sec:appro}. However, the solution can become highly non-smooth when approaching $t = 0$ if $g$ is non-differentiable and even fractal. For instance, the following result shows how complex the fractal landscape can be in the one-dimensional case:

\begin{proposition}
\label{th:landscape}
    [Theorem 1.2, \cite{posey}] Let $\eta \in C(\mathbb{R})$ be nowhere differentiable in $[a, b]$, then the set of its local maximum (minimum) is dense in $[a, b]$.
\end{proposition}
It suggests that in the case of fractal landscapes, there are so many local maximum points that it is challenging to determine where policy gradient converges towards. The gradient flow also loses stability as the policy approaches a deterministic limit, and even a small perturbation may result in a totally different trajectory in the policy space. 

In fact, reaching the initial condition $u(x, 0)$ from a future time $t = T > 0$ involves solving the backward heat equation \cite{renardy}:
\begin{equation}  \label{backward}
    \begin{cases}
      & 2 u_t + \Delta u = 0, \quad (x, t) \in \mathbb{R}^m \times (-\infty, T)\\
      & u = g_T, \quad (x, t) \in \mathbb{R}^m \times \{ T \}.
    \end{cases} \\
\end{equation}
where $g_T = u(x, T)$ is the solution at time $t = T$. It is well-known in PDE theory that \eqref{backward} is {\em ill-posed} \cite{kabanikhin}. The physical intuition is that we cannot reverse diffusion, which is an information-losing process over time. Indeed, a PDE system is called ill-posed if it does not have a solution, or the solution is not unique, or if the solution does not change continuously with respect to the initial (terminal) conditions.


To better understand the ill-posedness of \eqref{backward}, the following theorem states that an arbitrarily small perturbation of the terminal condition $g_T$ can preclude the existence of the solution:

\begin{theorem}[Ill-posedness around deterministic policies]
\label{th:nobound}
    For any $\sigma > 0$, terminal condition $g_{\sigma^2} \in L^2(\mathbb{R}^m)$ and $\epsilon > 0$, there exists $g'_{\sigma^2} \in L^2(\mathbb{R}^m)$ with $\| g_{\sigma^2} - g'_{\sigma^2} \|_{L^2(\mathbb{R}^m)} < \epsilon$ and the solution of \eqref{backward} does not exist for the terminal condition $u(\mu, \sigma^2) = g'_{\sigma^2}$.
\end{theorem}

The complete proof can found in Appendix~\ref{app:proof}. This result suggests that policy gradient methods are inherently ill-posed when reducing the variance. In this sense, even a small perturbation of the current state may significantly change the limit to which it converges.

\begin{figure*}[h!]
\centering

\subfigure[]{\includegraphics[width=0.24\textwidth]{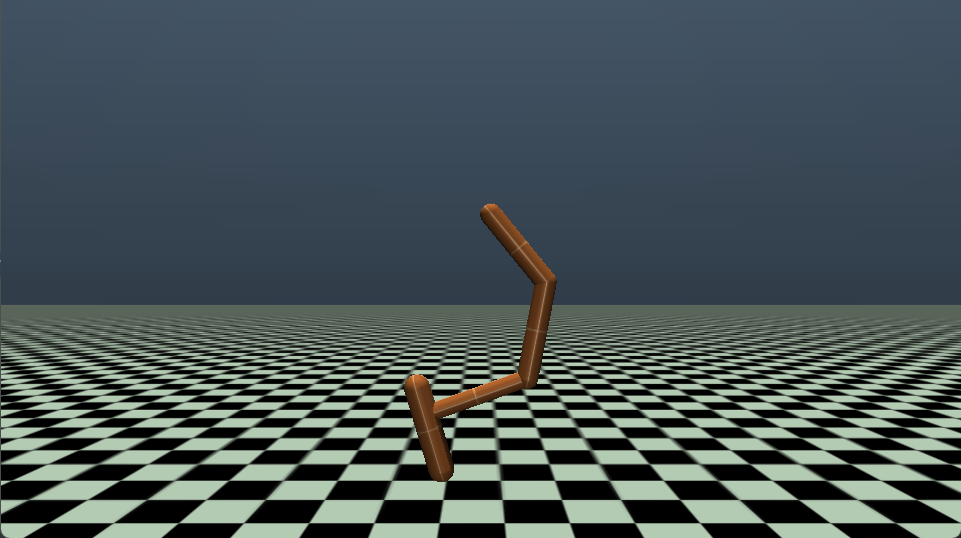}}
\subfigure[]{\includegraphics[width=0.24\textwidth]{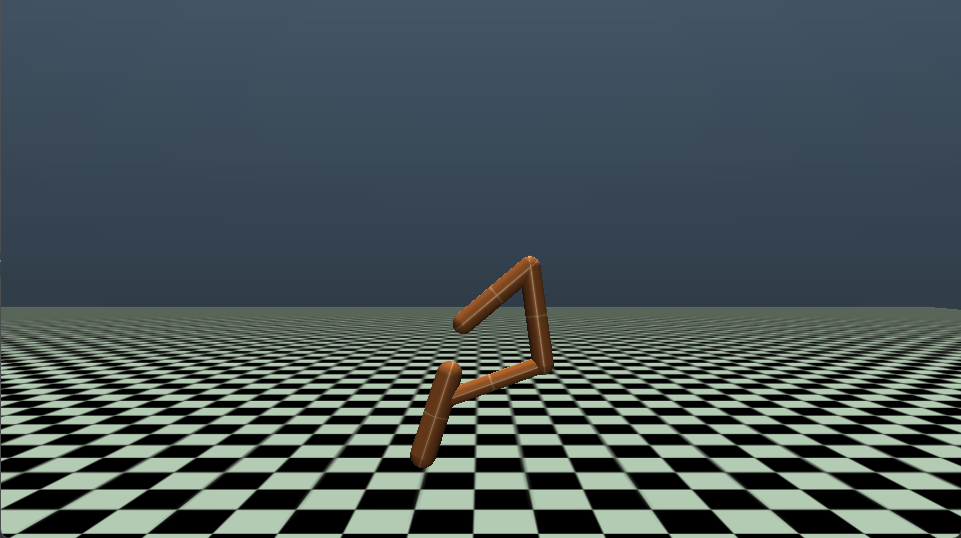}}
\subfigure[]{\includegraphics[width=0.24\textwidth]{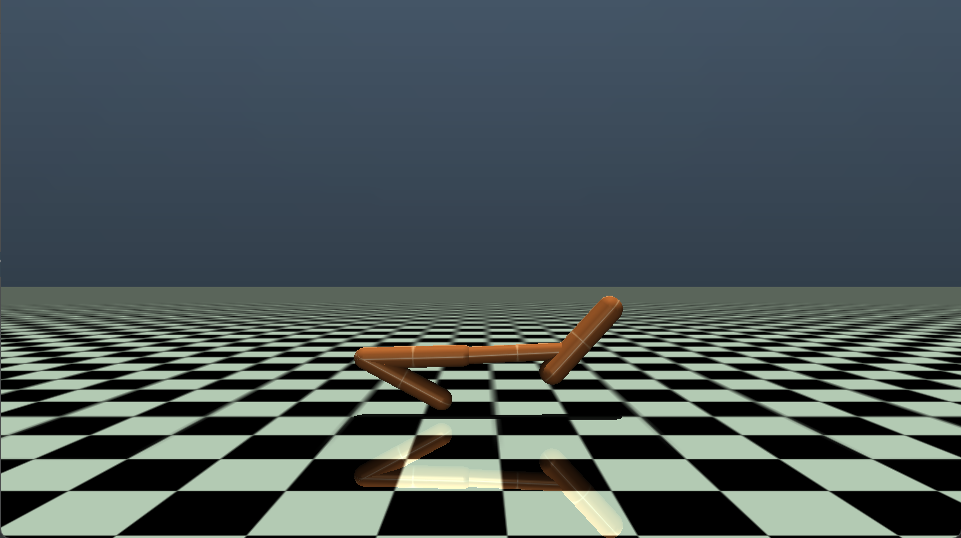}}
\subfigure[]{\includegraphics[width=0.24\textwidth]{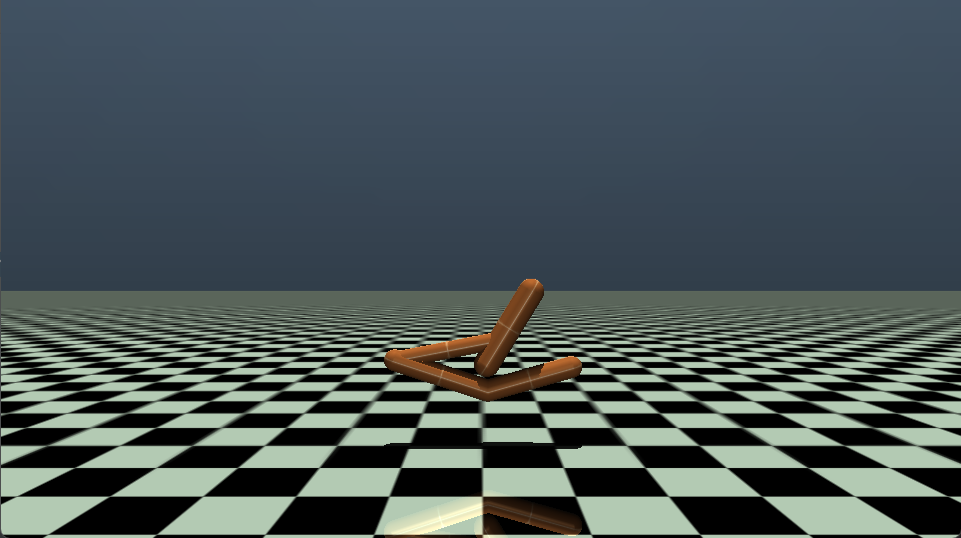}}

\caption{Hopper stand: the hopper failed to learn standing when $\sigma = 0.005$.}

\end{figure*}

\begin{figure*}[h!]
\centering

\subfigure[]{\includegraphics[width=0.24\textwidth]{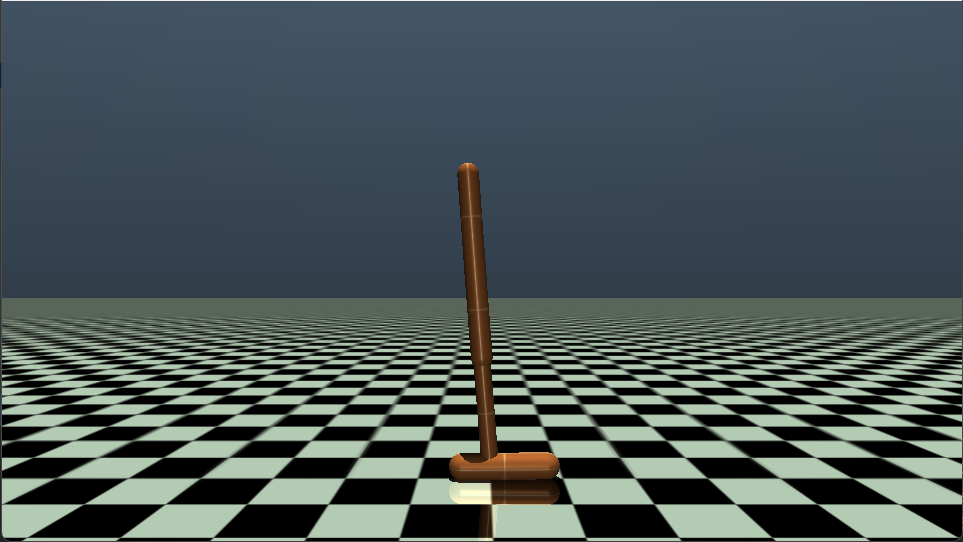}}
\subfigure[]{\includegraphics[width=0.24\textwidth]{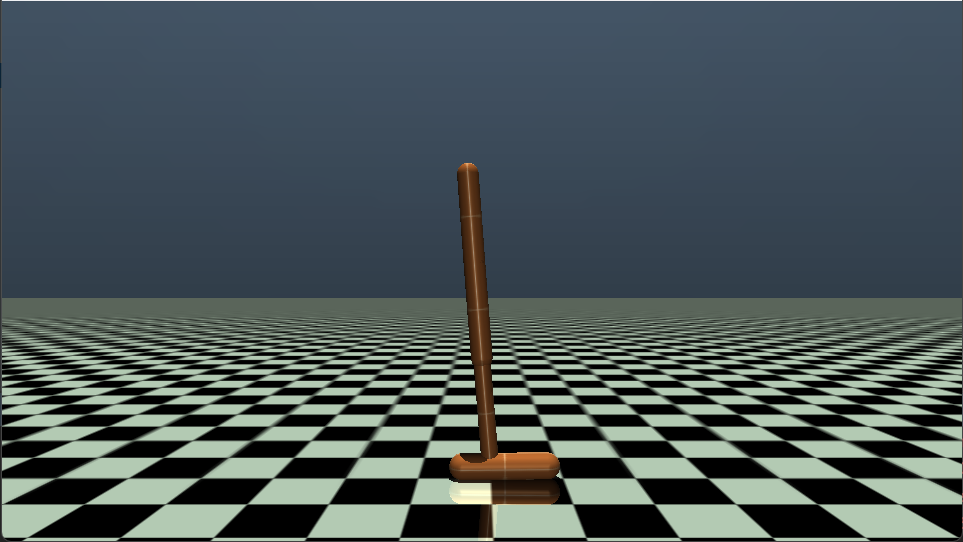}}
\subfigure[]{\includegraphics[width=0.24\textwidth]{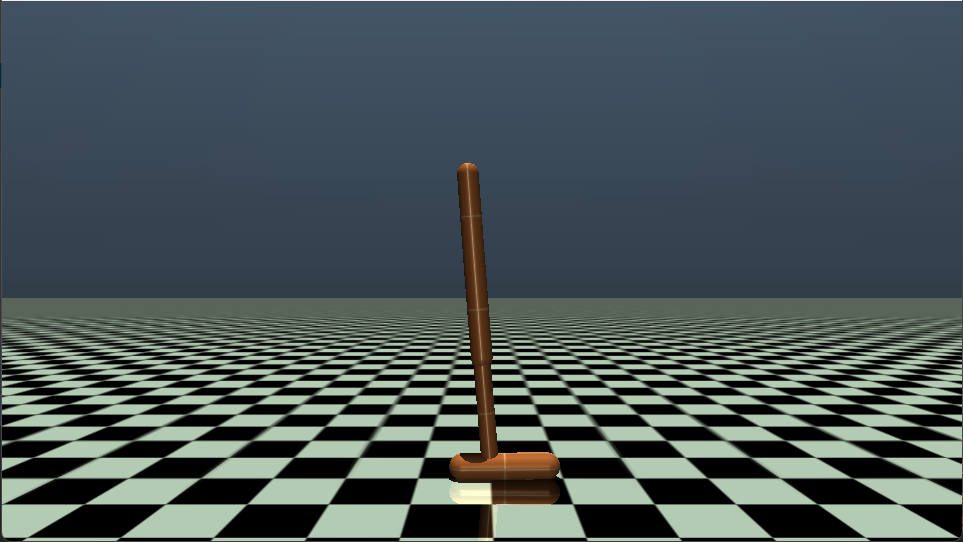}}
\subfigure[]{\includegraphics[width=0.24\textwidth]{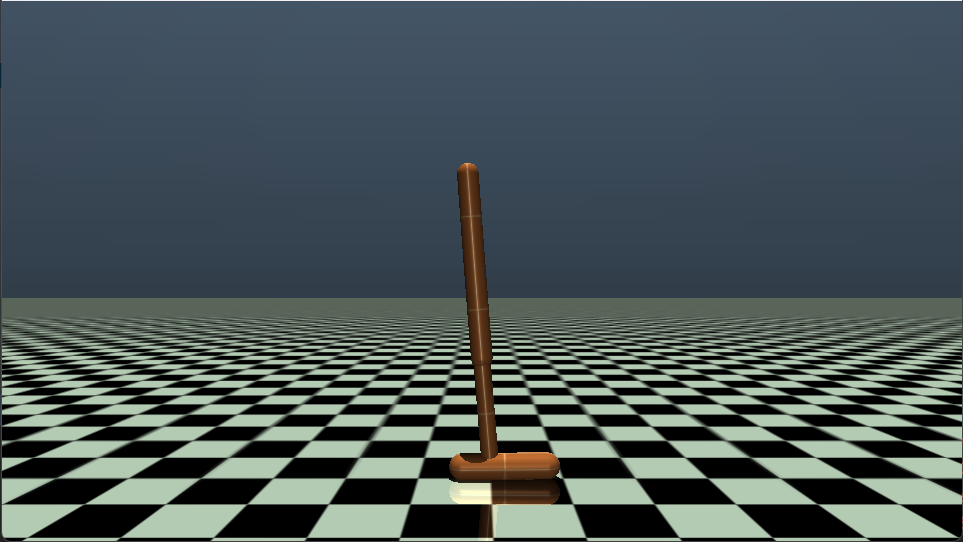}}

\caption{Hopper stand: the hopper successfully learned to stand when $\sigma = 0.05$.}

\end{figure*}

\paragraph{Existence of optimal variance.}


It has shown that as the variance $\sigma^2$ decreases, the mollified surrogate objective becomes less smooth and eventually converges to the fractal landscape of deterministic policy. On the other hand, note that the term $Q^{\pi}(s, a) \pi_\theta (a|s)$ is a random variable with $a \sim \pi_\theta(\cdot|s)$, the variance of $Q^{\pi}(s, a) \pi_\theta (a|s)$ will grows as well and the training process becomes more random as $\sigma$ increases. For instance, consider the gradient-ascent algorithm $X_{k + 1} = X_k + \delta G_k$ where $G_k$ is the estimated gradient and $\delta > 0$ is the stepsize. Suppose that we just obtained $X_k$, the conditional variance of the next parameter is given as $Var(X_{k+1} | X_k) = \delta^2 \ Var(G_k)$ (here we used $Var(X_k | X_k) = 0$), which means that the uncertainty in the gradient will propagate into the optimization parameters, and eventually affect the stability of the training curve. Combining these two facts together hints that there should exist an optimal value of $\sigma$ for the policy gradient and we conclude with the following assertion:
\begin{remark}
    For chaotic MDPs where the optimization landscapes are fractal, there exists an optimal variance $\sigma^*$ for the Gaussian policy that minimizes the uncertainty in training.
\end{remark}

\subsection{The uncertainty principle}
\label{sec:harmonic} 

The uncertainty principle, also known as the Heisenberg uncertainty principle, is a fundamental underlying law in harmonic analysis. It states that a function and its Fourier transform cannot be localized at the origin simultaneously. First, let us present the definition of Fourier transform:
\begin{definition}[Fourier transform] Let $\phi \in \mathcal{S}(\mathbb{R}^d)$ belong to the space of rapidly decreasing functions on $\mathbb{R}^d$ that consists of all indefinitely differentiable functions $f$ on $\mathbb{R}^d$ such that $\sup |x^\alpha (\frac{\partial}{\partial x})^\beta f(x)| < \infty$ for all multi-index $\alpha$ and $\beta$, then the Fourier transform of $\phi$ is defined as
    $$\mathcal{F}(\xi) = \int_{\mathbb{R}^d} \phi(x) e^{-2 \pi i \langle x, \xi \rangle} \ dx, \quad \xi \in \mathbb{R}^d,$$
    where $\langle x, \xi \rangle$ is the inner product of $x$ and $\xi$.
\end{definition}
An important property connecting Fourier transform and convolution is that $\mathcal{F}(\phi_1 * \phi_2) = \mathcal{F}(\phi_1) \mathcal{F}(\phi_2)$ for any $\phi_1, \phi_2 \in \mathcal{S}(\mathbb{R}^d)$. As mentioned in Section~\ref{sec:heat}, $\hat{\phi}$ describes the frequency and thus the non-smoothness of $\phi$. The rigorous formulation of the uncertainty principle is given as follows:


\begin{proposition}
    [Uncertainty Principle~\cite{stein}] If $\phi \in \mathcal{S}(\mathbb{R}^d)$ satisfies $\int_{\mathbb{R}^d} |\phi(x)|^2 \ \mathrm{d}x = 1$, then 
    \begin{equation}  \label{uncertainty}
        (\int_{\mathbb{R}^d} |x|^2 |\phi(x)|^2 \ \mathrm{d}x) (\int_{\mathbb{R}^d} |\xi|^2 |\hat{\phi}(\xi)|^2 \ \mathrm{d}\xi) \geq \frac{d^2}{16 \pi^2},
    \end{equation}
    where $\hat{\phi}$ is the Fourier transform of $\phi$ and $\mathcal{S}(\mathbb{R}^d)$ denotes the space of rapid decreasing functions. The equality holds when $\phi$ is a Gaussian function.
\end{proposition}

To see how it is related to the mollification, let us consider the probability density function $\phi(x) \in \mathcal{S}(\mathbb{R}^d)$. Let $g(a) = Q^\pi(s_0, a)$ for simplicity, the convolution $g * \phi$ is close to $g$ when $\phi$ concentrates at $x = 0$, i.e., the quantity $\int_{\mathbb{R}^d} |x|^2 |\phi(x)|^2 \ \mathrm{d}x$ is small. On the other hand, $g * \phi$ is smooth when $\hat{\phi}$ concentrates at $\xi = 0$, which is equivalent to having small $\int_{\mathbb{R}^d} |\xi|^2 |\hat{\phi}(\xi)|^2 \ \mathrm{d}\xi$. However, the uncertainty principle prohibits us to achieve these two things at the same time, which leads to a fundamental limitation: when the policy gradient mollifies and explores the landscape, it inevitably increases the risk of over-smoothing in the meantime. In particular, if the region of attraction of the optimal policy is small, the Gaussian kernel in policy gradient may completely eliminate that region and the problem will become unsolvable (see \emph{Planar quadrotor balance} experiment in Section~\ref{sec:experiment} for details). Therefore, policy gradient methods are inherently limited when reducing the variance to zero, fundamentally challenging their applications in control and robotics.

\section{Experiments}
\label{sec:experiment}

In this part, we will apply the theory established in previous sections to explain when policy gradient methods can/cannot solve certain control problems. The controls are linearly parameterized as $u = -Ks + b$ in the planar quadrotor task, and $u = -Ks$ in the double pendulum example. The optimal solution to those problems are obtained using the LQR method in optimal control. In the hopper stand task, we use a 2-layer neural network as controller. More details of experiments are provided in Appendix~\ref{app:detail} and \ref{sec:behavior}.

\paragraph{Hopper stand.}

We begin with a standard example in the OpenAI GYM documentation \citep{brockman}. As shown in Figure~\ref{fig:hopper}(e), the policy landscape for a randomly initialized policy is fractal due to the chaoticness in the underlying dynamics. To penalize any deviation from the balanced standing position, we apply a negative quadratic cost so that the total reward is always non-positive. Here we conduct 4 sets of parallel experiments with different standard deviations for the Gaussian policy, and each set is repeated across five random seeds. 

In Figure~\ref{fig:hopper} (a), the resulted variance of the return sequence $\{J(\theta_k) \}_{k = 0}^{100}$ is even greater when $\sigma = 0.005$. This is attributed to the weak mollification effect, resulting in an optimization landscape very close to the fractal landscape generated by the deterministic policy. 
It shows that both $\sigma = 0.05$ and $\sigma = 0.5$ achieve nice and stable performance than the other two standard deviations, which agrees with the statement in Section~\ref{sec:uncertainty} that the variance of stochastic policies cannot be too large or too small. The behaviors of the hopper in simulation are presented in Appendix~\ref{sec:behavior}, and we can observe that the policy gradient with a standard deviation of $\sigma = 0.05$ discovers a policy that successfully stabilizes the hopper.

In Figure~\ref{fig:hopper} (f), after 100 epochs, the trained policy entered into a smooth region which corresponds to a stable dynamics of the hopper, compared to the fractal landscape around the initial policy $\theta_0$ in (e). It demonstrates that policy gradient is capable of guiding chaotic environments towards stability in some tasks with a proper sampling variance. 


\begin{figure}[h!]
\centering

\subfigure[]{\includegraphics[width=0.23\textwidth]{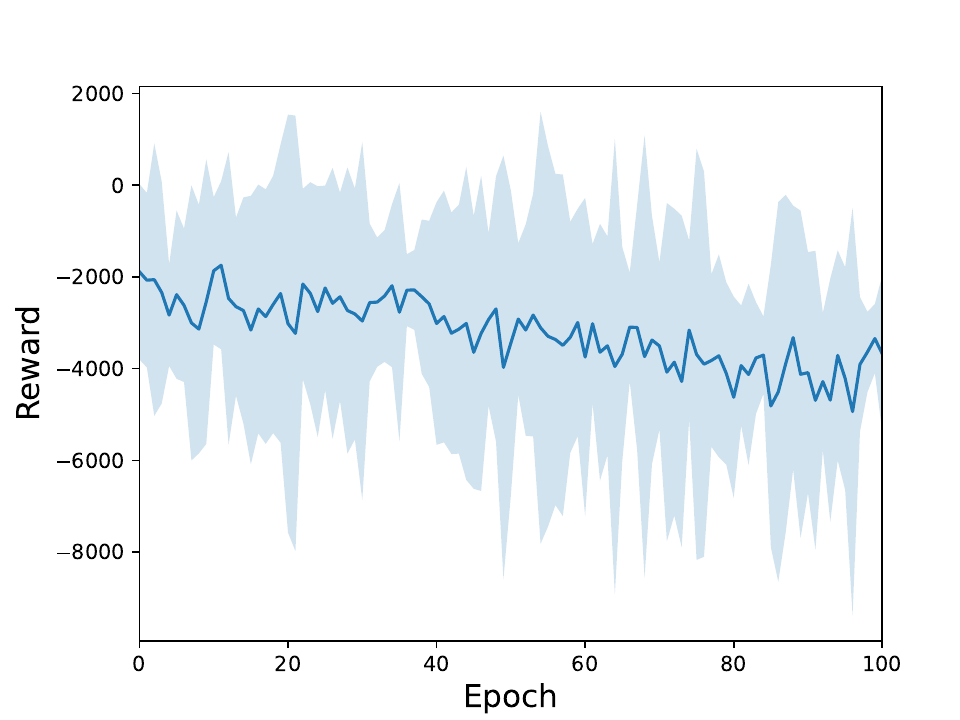}}
\subfigure[]{\includegraphics[width=0.23\textwidth]{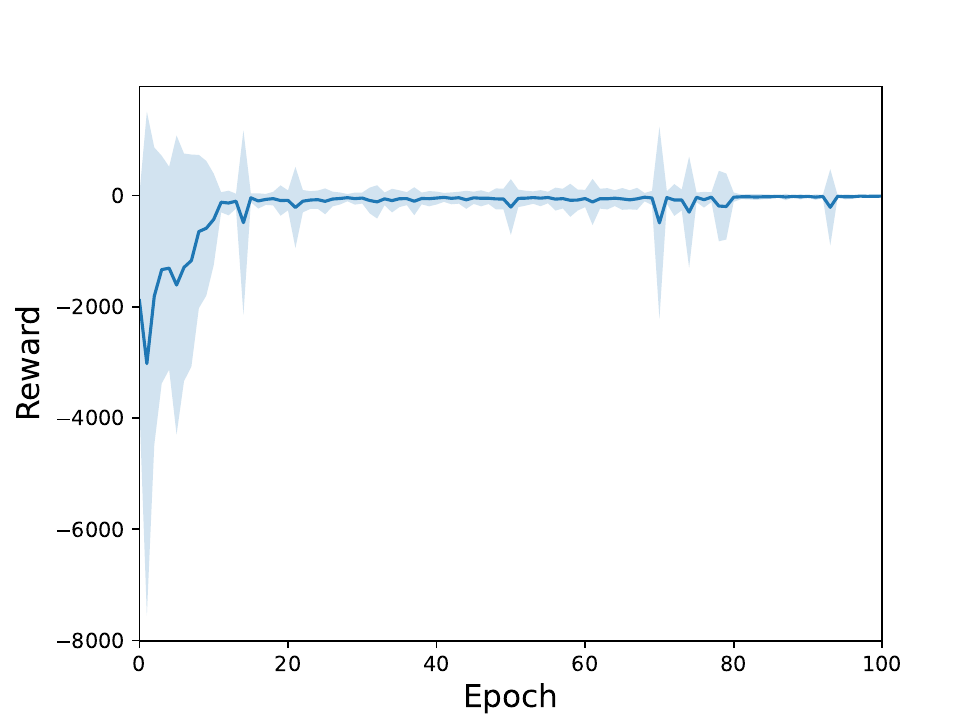}}
\subfigure[]{\includegraphics[width=0.23\textwidth]{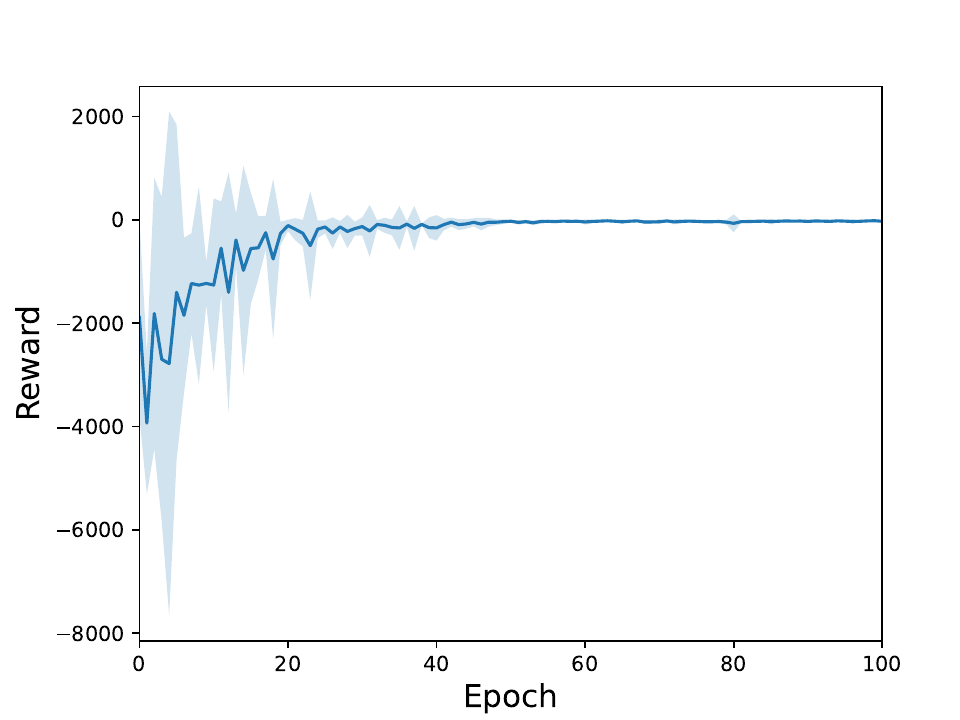}}
\subfigure[]{\includegraphics[width=0.23\textwidth]{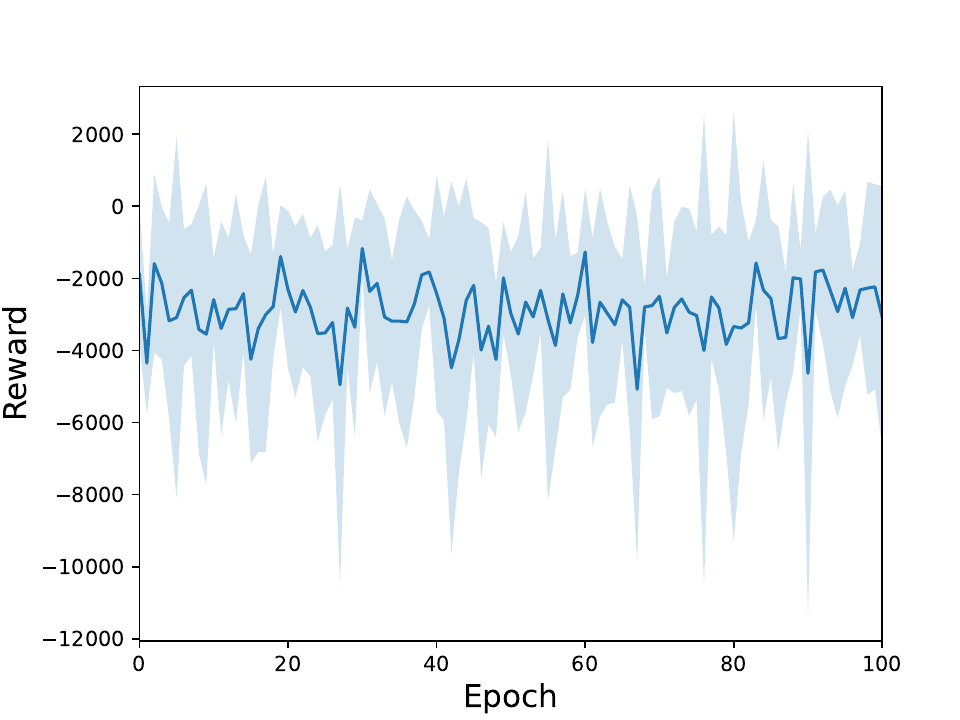}}

\subfigure[]{\includegraphics[width=0.23\textwidth]{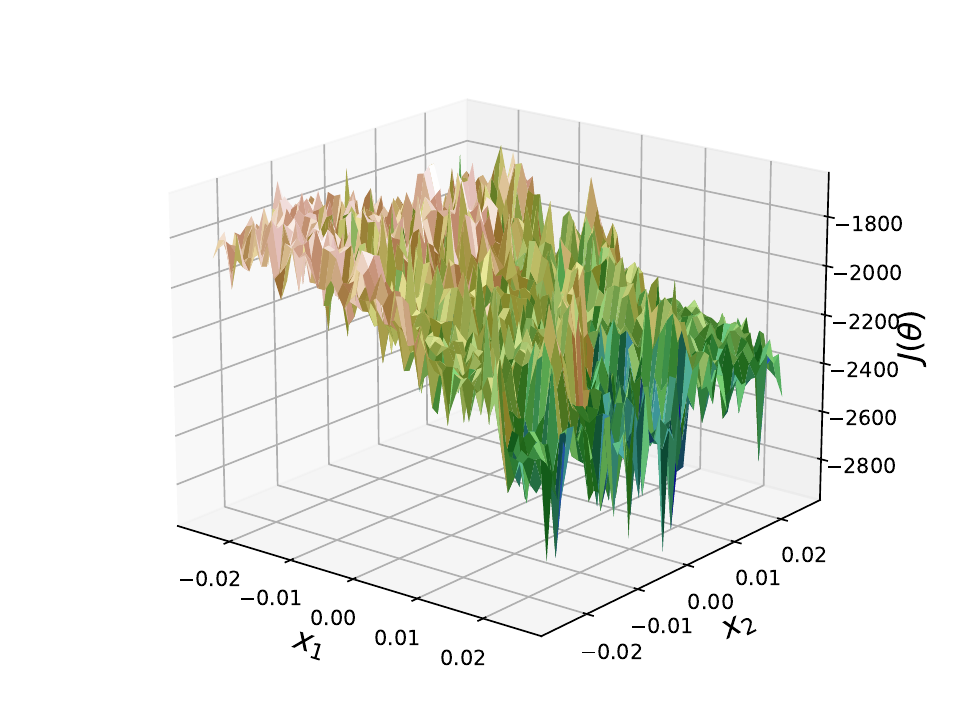}}
\subfigure[]{\includegraphics[width=0.23\textwidth]{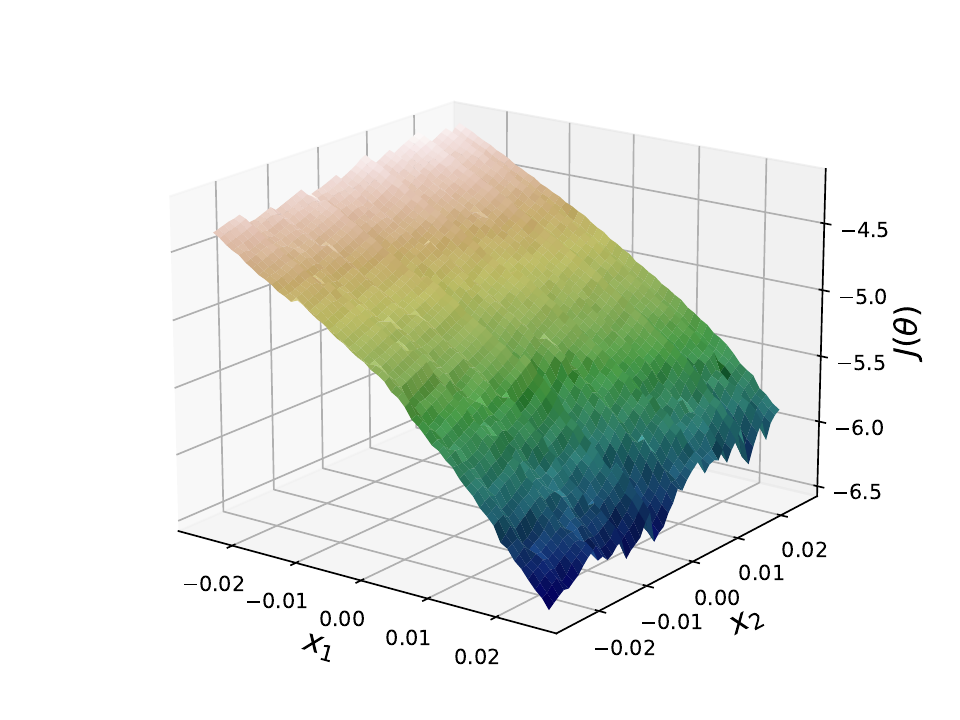}}

\caption{Hopper: Training curves with different standard deviations: (a) $\sigma = 0.005$; (b) $\sigma = 0.05$; (c) $\sigma = 0.5$; (d) $\sigma = 5$. The $x$-axis is the number of epochs and the $y$-axis is the total reward. We can see that training is successful when the variance is neither too small nor too large ($\sigma = 0.05$ and $0.5$); (e) policy landscapes around the initial policy $\theta_0$; (f) policy landscapes around the final policy $\theta_{100}$. When generating the landscapes, $\sigma = 0.05$ is employed and we use the deterministic version of the policy for plotting.}
\label{fig:hopper}

\end{figure}

\paragraph{Double pendulum stabilization.}

It is well-known that the double pendulum system can exhibit chaotic behavior which leads to a fractal optimization landscape. In this experiment, we adopt the dynamics from \cite{chang} and the initial policy $\theta_0 = [K_0, b_0]^T$ is selected to be close to the stabilizing region, but still generates an unstable trajectory. The initial state $s_0 = [-0.2, 0.2, 0, 0]^T$ and the reward function is quadratic. As depicted in Figure~\ref{fig:dp} (a), the trained policy successfully stabilizes the system towards the origin (upright position) after 50 epochs. As in Figure~\ref{fig:dp} (b), the $Q$-function landscape of the trained policy $\theta_{50}$ is smooth, indicating that the system is well-behaved. This example suggests a scenario where policy gradient methods can exhibit its full strength: fine-tuning an initial policy that is close enough to the desired policies.

\begin{figure}[h!]
\centering

\subfigure[]{\includegraphics[width=0.23\textwidth]{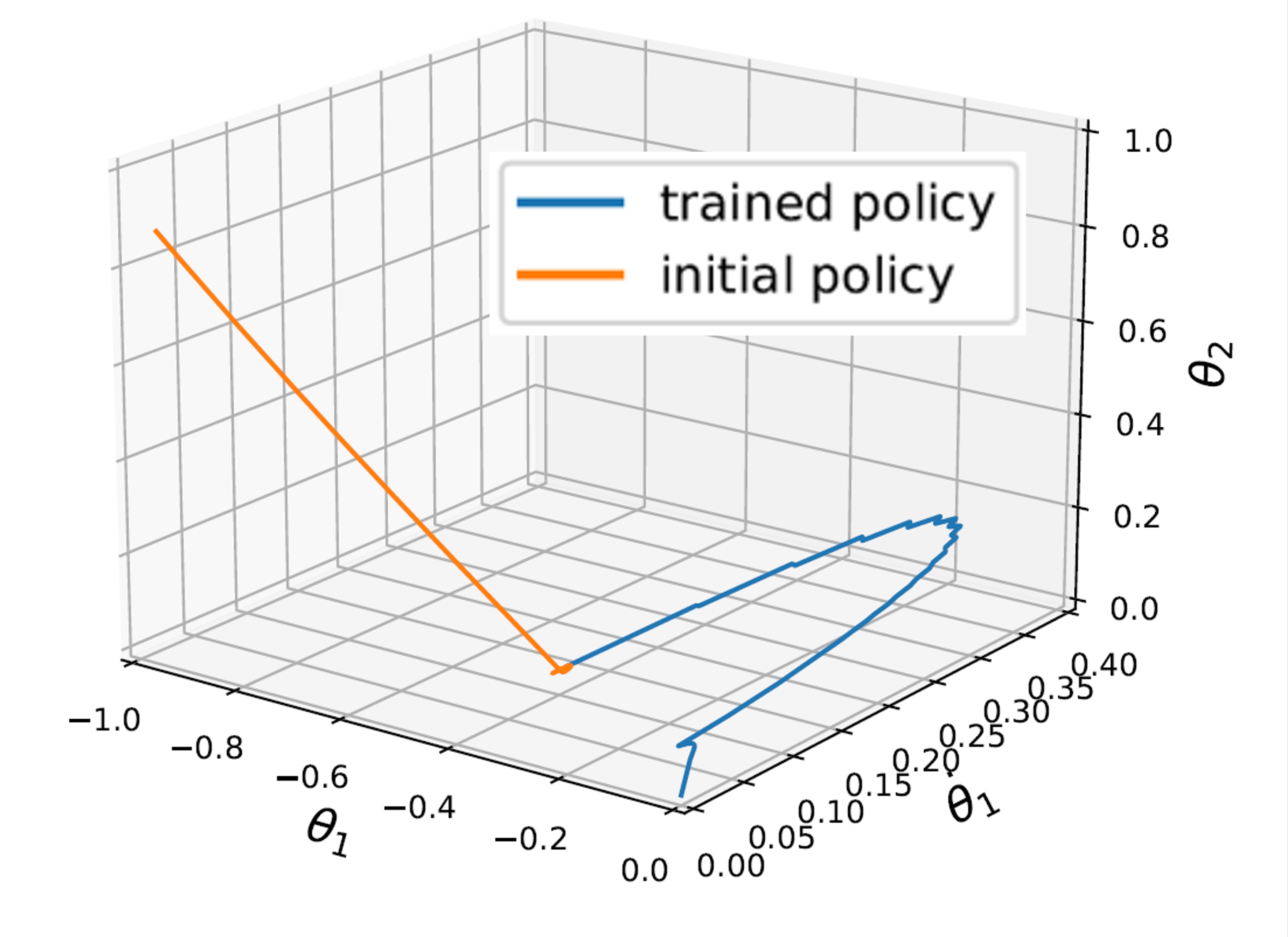}}
\subfigure[]{\includegraphics[width=0.23\textwidth]{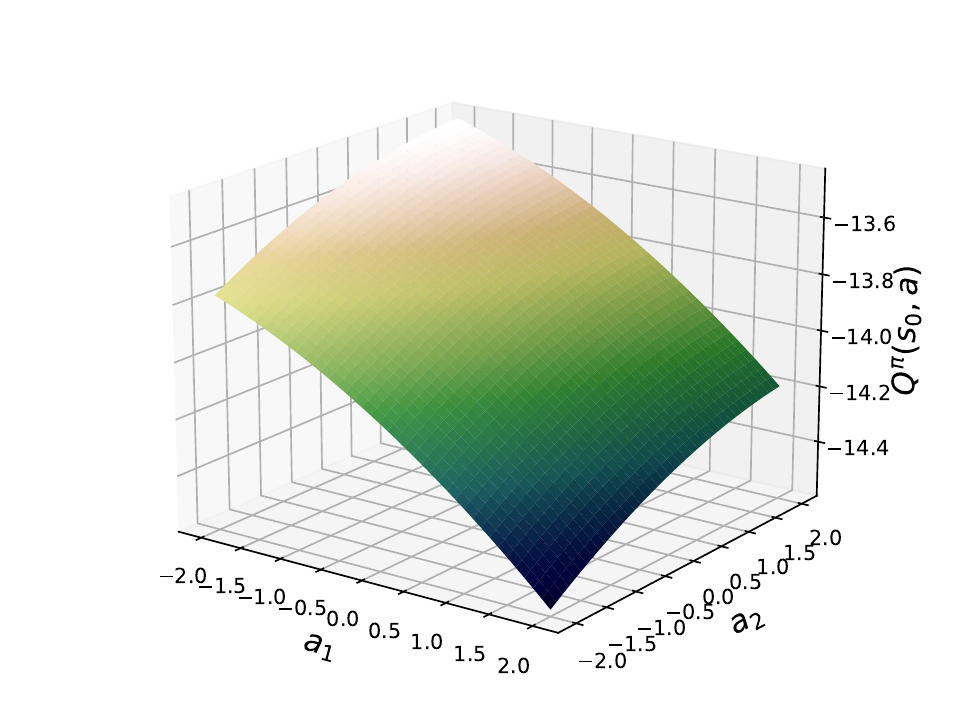}}

\caption{Double pendulum: (a) the trained policy successfully stabilizes the double pendulum system towards the origin, while the initial policy does not; (b) the  $Q$-landscape $Q^{\pi}(s_0, a)$ of the trained policy in action space.}
\label{fig:dp}

\end{figure}

\paragraph{Planar quadrotor balance.}

Now let us evaluate the hardness of the balance task of the planar quadrotor system \cite{tedrake}, and we will understand it from the view of mollification effects. Consider the control system where the two control inputs $u_1, u_2$ are the propelling force provided by motors. The main difficulty of balancing the quadrotor horizontally via RL algorithms is that the policy must learn to generate equal control inputs at two propellers, otherwise the quadrotor will immediately flip and fall. Therefore, even a slight deviation from the balanced control results in a significant divergence in the trajectory, which means that the landscape around the optimal solution is highly non-smooth as shown in Figure~\ref{fig:drone} (a) despite the dynamics of quadrotor is \emph{not chaotic}. In particular, the objective landscape has a spike at the origin (the optimal policy), which behaves like a high-frequency signal in contrast to the surrounding landscape.

\begin{figure}[h!]
\centering

\subfigure[]{\includegraphics[width=0.23\textwidth]{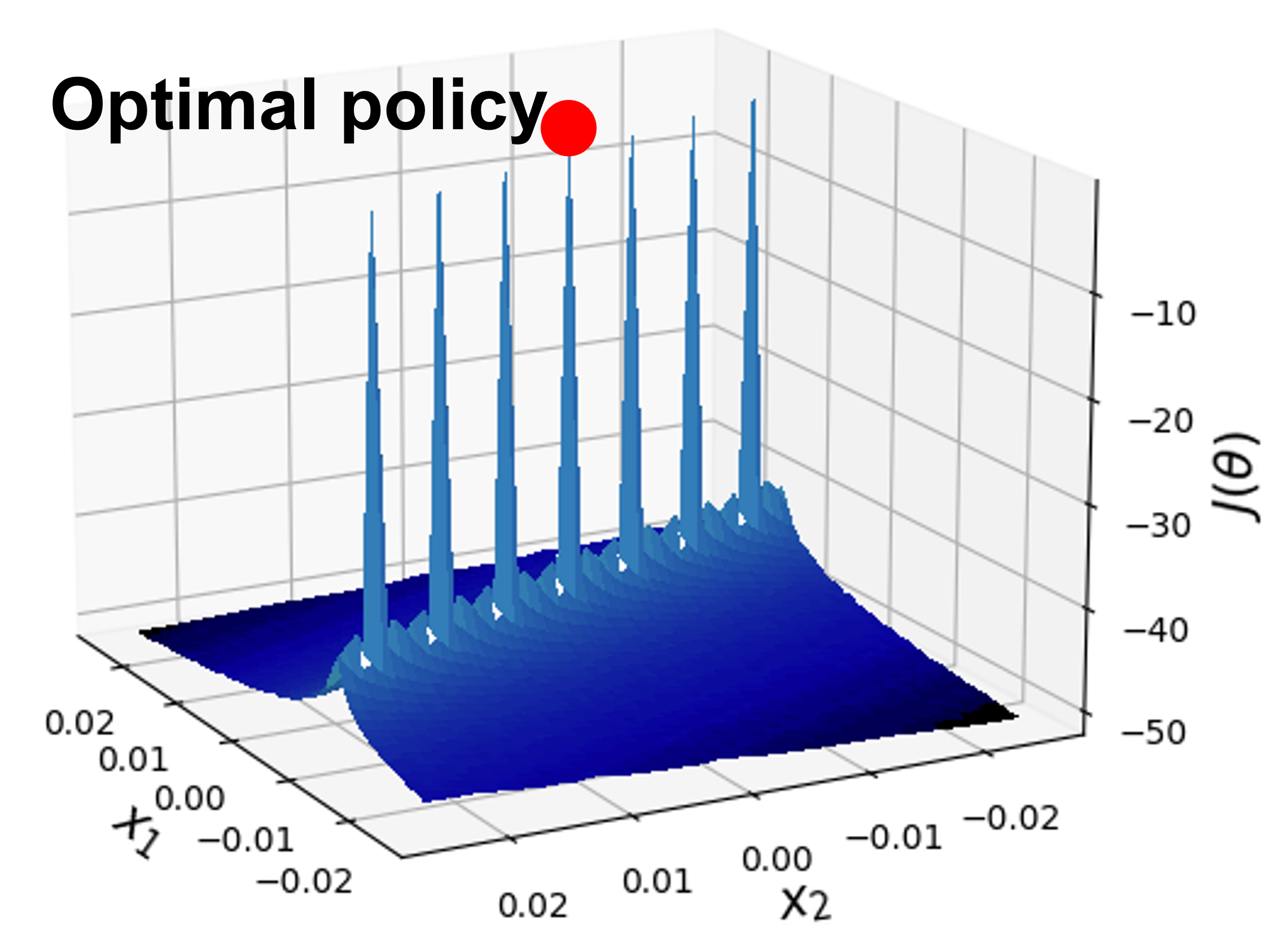}}
\subfigure[]{\includegraphics[width=0.23\textwidth]{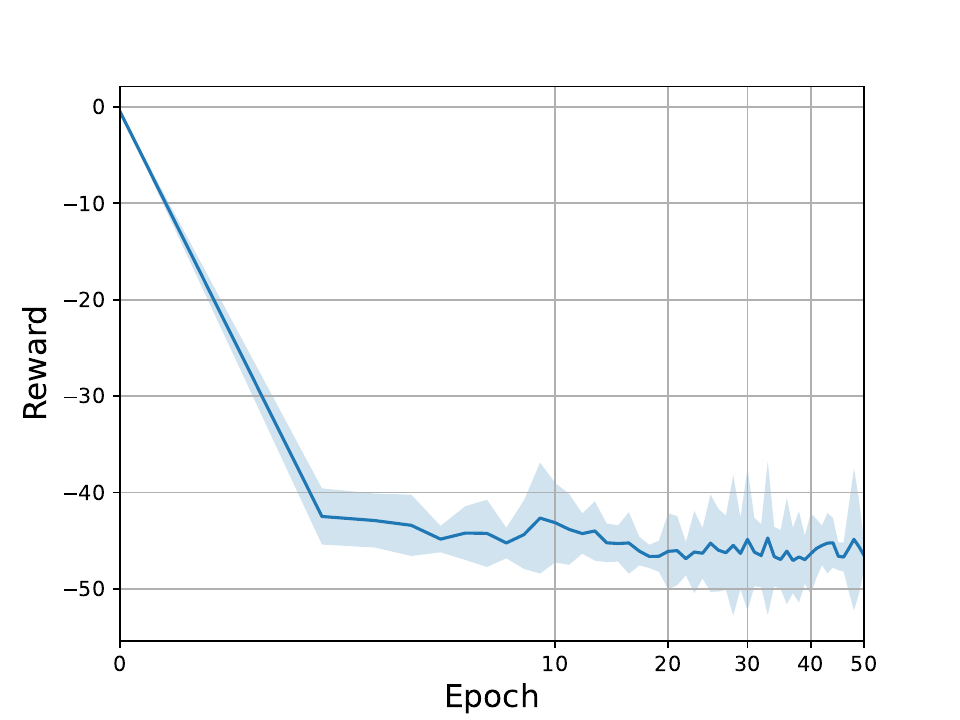}}

\caption{Planar quadrotor: (a) the policy landscape around the optimal policy $\theta^*$ is a spike, which is then easily eliminated by the Gaussian kernel in the policy gradient; (b) the training curve immediately falls after one epoch even the initial policy has already been optimal.}
\label{fig:drone}

\end{figure}

As a consequence, even if we run the policy gradient algorithm with the initial policy $\theta_0 = \theta^*$ (the optimal policy), the training curve immediately drops after one epoch and never returns to the optimal policy $\theta^*$ as shown in Figure~\ref{fig:drone} (b). This phenomenon is explained by the theoretical results in Section~\ref{sec:uncertainty} that the optimal solution is a high-frequency noise and is filtered by the Gaussian kernel in policy gradient. Traditional control methods presume that the torque applied to the two motors must be equal, while model-free RL algorithms have to figure this out via trial and error, which leads to the aforementioned problem.

\section{Concluding Remarks}






A theoretical framework for understanding the mollification effect of policy gradient methods is proposed in this paper. From this perspective, we demonstrate that the mollification effect of policy gradient can be both advantageous and a bottleneck, depending on the practical cases in which it is applied and how the algorithm is implemented. Given that there are some control problems where policy gradient can almost filter the true solution as high-frequency noise, its strength should not be overestimated. Hence, a systematic understanding of the role RL should play in control and robotics needs to be established in the future.

\section*{Acknowledgements}

This material is based on work supported by NSF Career CCF 2047034, NSF CCF 2112665 TILOS AI Institute, NSF CCF DASS 2217723, ONR YIP N00014-22-1-2292, and Amazon Research Award. We thank the anonymous reviewers for their valuable suggestions.

\section*{Impact Statement}

This paper presents work to advance the theoretical understanding of the mechanisms and limitations of deep reinforcement learning algorithms. We believe the potential societal consequences are positive, as they will guide more principled use of machine learning techniques.

\bibliography{references}
\bibliographystyle{icml2023}

\newpage
\appendix
\onecolumn

\section{Convolution}
\label{app:fourier}

We introduce some fundamental concepts in Fourier analysis. Let $L^1(\mathbb{R}^N) = \{ f: \mathbb{R}^N \rightarrow \mathbb{R}^N: \int_{D} |f| < \infty  \}$ denote the space of integrable functions, and let $L^\infty(\mathbb{R}^N) = \{ f: \mathbb{R}^N \rightarrow \mathbb{R}^N: |f(x)| \leq M \textit{ a.e. x } \in \mathbb{R}^N \textit{ for some } M > 0 \}$ denote the space of essentially bounded measurable functions.

\paragraph{Notation.}

To avoid any possible ambiguity, we collect the notations below for later reference. Unless mentioned otherwise, let $U \subset \mathbb{R}^N$ be an open set, and we have:

\begin{itemize}

    \item $C_c^\infty(U)$: The space of infinitely differentiable functions $\phi: U \rightarrow \mathbb{R}$ with compact support in $U$ (i.e., the set $ \{ \theta \in \mathbb{R}^N : \phi(\theta) \neq 0  \}$ is a compact subset of $U$).

    \item $L^1_{loc}(U)$: The space of locally integrable functions on $U$ ($v \in L^1_{loc}(U)$ means that for any compact set $K \subset U$, $\phi$ is integrable on $K$).

    \item $\mathcal{B}(x, r)$: The open ball of radius $r > 0$ centered at $x$.

    \item $\mathcal{S}(\mathbb{R}^d)$: The space of rapid decreasing functions that consists of all indefinitely differentiable functions $f$ on $\mathbb{R}^d$ such that $\sup |x^\alpha (\frac{\partial}{\partial x})^\beta f(x)| < \infty$ for all multi-index $\alpha$ and $\beta$.

\end{itemize} 

The convolution operator is defined as:
\begin{definition}[Convolution]
 Suppose that $\phi_1 \in L^1(\mathbb{R^N})$ and $ \phi_2 \in L^{\infty}(\mathbb{R^N})$, then their convolution $\phi_1 * \phi_2 \in L^{\infty}(\mathbb{R^N})$ is defined by
    $$(\phi_1 * \phi_2) (x) = \int_{\mathbb{R^N}} \phi_1(y) \phi_2 (x - y) \ dy.$$
\end{definition}
An important property of convolution is that for any variable $x_i$, the partial derivative satisfies
$$\frac{\partial (\phi_1 * \phi_2)}{\partial x_i} =  \frac{\partial \phi_1}{\partial x_i} * \phi_2 = \phi_1 * \frac{\partial \phi_2}{\partial x_i}$$
which implies that $\phi_1 * \phi_2$ is smooth as long as one of $\phi_1$ and $\phi_2$ is smooth. In particular, we have $\nabla (\phi_1 * \phi_2) = \phi_1 * \nabla \phi_2$ when $\phi_2$ is smooth. 

\section{The Initial State Distribution $\rho_0$} 
\label{app:initial}

Here we briefly discuss why the distribution of initial states does not affect the smoothness of the original objective too much. To demonstrate this, let us rewrite the value function $V^{\pi_\theta}(s)$ as $V(s; \theta)$ to emphasize its dependence on $\theta$. Then, the objective function is given by
$$J(\theta) = \int V(s; \theta) \rho_0(s) \ \mathrm{d}s.$$
Thus, if $J(\theta)$ is differentiable, we will be able to push the differentiation inside the integral as $\rho_0$ is independent of $\theta$, i.e.,
\begin{equation}
\label{eq:swap}
    \frac{\mathrm{d} J}{\mathrm{d} \theta} = \int \frac{\partial V(s; \theta)}{\partial \theta} \rho_0(s) \ \mathrm{d}s.
\end{equation}
Therefore, the smoothness of $J(\theta)$ is closely associated with the differentiability of $V(s, \theta)$ \emph{regardless of} the distribution $\rho_0$. On the other hand, if the value function $V(s; \theta)$ is not differentiable in $\theta$-space for all $s \in S$ where $S \subset \mathbb{R}^n$ is a set of positive measure, $\frac{\mathrm{d} J}{\mathrm{d} \theta}$ no longer exists as the integral on the right-hand side of \eqref{eq:swap} diverges.

\section{Proof of Theorem~\ref{th:nobound}}
\label{app:proof}

\begin{proof}
    Without loss of generality, we assume $m = 1$. Since heat equations are linear, it suffice to prove the case of $g_T \equiv 0$ where $T = \sigma^2$.

    Note that for any terminal condition $\phi \in L^2(\mathbb{R})$, the solution of \eqref{backward} is given by 
    $$u'(x, t) = \frac{1}{\sqrt{2 \pi (T - t)}} \int_{\mathbb{R}} \phi(z)  e^{\frac{| x - z |^2}{2 (T - t)} } \ \mathrm{d}z.
    $$
    Now let $g'_T(x) = \frac{\epsilon}{2} e^{-\frac{|x|}{2}}$ so that $\| g'_T \| = \frac{\epsilon}{\sqrt{2}} < \epsilon$, and for any $t < T$ and $x \in \mathbb{R}$, we have
    \begin{align*}
         \frac{1}{\sqrt{2 \pi (T - t)}} \int_{\mathbb{R}} g'_T(z)  e^{\frac{| x - z |^2}{2 (T - t)} } \ \mathrm{d}z &= \frac{\epsilon}{2} \frac{1}{\sqrt{2 \pi (T - t)}} \int_{\mathbb{R}}  e^{\frac{| x - z |^2}{2 (T - t)} -\frac{|z|}{2}} \ \mathrm{d}z \\
        &\geq \frac{\epsilon}{2} \frac{1}{\sqrt{2 \pi (T - t)}} \int_{0}^\infty  e^{\frac{(z - x)^2}{2 (T - t)} -\frac{z}{2}} \ \mathrm{d}z \\
        &\geq \frac{\epsilon}{2} \frac{1}{\sqrt{2 \pi (T - t)}} \int_{0}^\infty  e^{\frac{x^2 - (x + \frac{T - t}{2})^2}{2 (T - t)}} \ \mathrm{d}z \\
        &= \infty,
    \end{align*}
    which implies that the solution $u'(x, t)$ does not exist and we complete the proof.
\end{proof}

\section{Lipschitz Continuous Objectives}  \label{sec:appro}

In Section~\ref{sec:uncertainty}, we have seen that the existence of $\lim_{t \rightarrow 0} u(x, t)$ depends on the smoothness of initial condition $g$. However, for many problems such as finite state-space MDPs in RL, the objective function is locally Lipschitz continuous. In this section, we will first introduce the notion of weak derivatives from the distribution theory \citep{rudin}, then prove the gradient estimated by policy gradient methods converges to the weak derivative of objective function. Consider the mollified objective function
\begin{equation}   \label{sg1}
    F(\mu, \sigma) = f \ast p_\sigma (\mu).
\end{equation}
where $p_\sigma(\mu)$ is the Gaussian distribution with mean $\mu$ and covariance $\sigma^2 \mathcal{I}$. The definition of weak derivatives that extends the notion of conventional derivatives:
\begin{definition}
  (Weak derivative)  Assume $U \subset \mathbb{R}^N$ is open. Suppose that $u, v \in L^1_{loc}(U)$ and $\beta \in \Lambda(N)$. We say that $v$ is the $\beta^{th}$-weak partial derivative of $u$, written as $D^\beta u = v$, if
    $$\int_{U} u D^\beta \phi \ dx = (-1)^{|\beta|} \int_U v \phi \ dx$$
    for all $\phi \in C^\infty_{c} (U)$, where $D^\beta \phi$ gives the corresponding $\beta^{th}$-conventional partial derivative of $\phi$ and $\Lambda(N) = \mathbb{Z}^N$ is the set of multi-indices of dimension $N$, that is,  every $\beta = (a_1, ..., a_N) \in \Lambda(N)$, $|\beta| = a_1 + ... + a_N$.
\end{definition}

The following example gives an idea of how weak derivatives are related to strong derivatives:

\begin{example}
    Consider the function $u(x) = |x|$ where $x \in U = (-1, 1)$, and define
    $$v(x) = \begin{cases}
      & 1, \quad x \geq 0\\
      & -1, \quad x < 0.
    \end{cases} \\$$
    Now let us show $Du = v$ using the preceding definition, that is, for any $\phi \in C_c^\infty(U)$, we need to prove
    $$\int_{-1}^1 u \phi' \ dx = - \int_{-1}^1 v \phi \ dx = \int_{-1}^0 \phi \ dx - \int_{0}^1 \phi \ dx .$$
    Applying the integration by parts, it yields
    \begin{align*}
        \int_{-1}^1 u \phi' \ dx &= -\int_{-1}^0 x \phi' \ dx  + \int_{0}^1 x \phi' \ dx  \\
        &= \phi(1) - \int_{0}^1 \phi \ dx - \phi(-1) + \int_{-1}^0 \phi \ dx \\
        &= \int_{-1}^0 \phi \ dx - \int_{0}^1 \phi \ dx \\
        &= - \int_{-1}^1 v \phi \ dx
    \end{align*}
    where we use the fact that $\phi \in C_c^\infty(U)$ (thus $\phi(-1) = \phi(1) = 0$) at the third equality.
\end{example}

Then we will see that the gradient estimated by policy gradient methods converges to the weak gradient of the true objective function as $\sigma \rightarrow 0$ when $f(\theta)$ is globally Lipschitz continuous:



\begin{theorem}
\label{th:convergence}
    Suppose that $f$ is globally Lipschitz in $U \subset \mathbb{R}^N$ and uniformly bounded over $\mathbb{R}^N$ where $U \subset \mathbb{R}^N$ is a bounded open set. Let $D^{\alpha_i} u$ denote the first-order derivative with respect to $x_i$, then 
    \begin{itemize}
        \item $f * p_{\sigma} \rightarrow f$, a.e.;

        \item  $ \frac{\partial (f * p_{\sigma})}{\partial x_i} \rightarrow D^{\alpha_i} f$, a.e.;
    \end{itemize}
    as $\sigma \rightarrow 0$.
\end{theorem}

Indeed, it explains why policy gradient methods work sufficiently well in finite-space problems \citep{agarwal}. On the other hand, if the objective function is locally Lipschitz continuous but not uniformly bounded, the convergence is not guaranteed as the function diverges fast as $\| \theta \| \rightarrow \infty$ as shown in the following example:
\begin{example}
    Consider the function $f(x) = e^{x^3}$ and $p_\sigma(x) = \frac{1}{\sigma \sqrt{2 \pi}}e^{-\frac{x^2}{2 \sigma^2}}$ is a Gaussian kernel. where $x \in \mathbb{R}$ and $\sigma > 0$. Apparently, $f$ is locally Lipschitz continuous. However, the integral
    $$\int_{-\infty}^\infty f(x) \ p_\sigma (y - x) \ dx =  \frac{1}{\sigma \sqrt{2 \pi}} \int_{-\infty}^\infty 3 x^2 e^{(x^3 -\frac{(y - x))^2}{2 \sigma^2})} \ dx = \infty.
    $$
    Thus, $f * p_\sigma$ does not exist for all $\sigma > 0$.
\end{example}
It indicates a fundamental drawback of Gaussian kernel: even though the objective function $f$ is smooth in a neighborhood of every policy parameter $\theta$, the mollification $f * p_\sigma$ may blow up if $f$ diverges exponentially fast at infinity. Therefore, a local gradient estimator may be a better choice in this case:
\begin{definition}
    (Bump function) For any $\sigma > 0$, let $\eta_{\sigma} \in C^\infty(\mathbb{R}^N)$ be defined by 
    \begin{equation}
        \eta_{\sigma}(x) = 
    \begin{cases}
      & C \exp(\frac{1}{|x|^2 - \sigma^2}), \quad |x | \leq \sigma\\
      & 0, \quad |x| > \sigma.
    \end{cases} \\
    \end{equation}
    where the constant $C > 0$ is selected so that $\int_{\mathbb{R}^N} \eta_{\sigma} dx = 1$.
\end{definition}

Unlike Gaussian kernel which uses information from the entire space, bump function samples data from a small neighborhood of $\theta_0$. However, as we mentioned in previous sections, Gaussian kernels suffer from two problems: first, the computational budget may not allow one to sample trajectories as many as they want, which leads to a significant sample bias when the variance is large; second, as a non-local estimator, Gaussian kernel samples data from the entire space so that the region of high loss/frequency may affect the estimate of local behavior. Therefore, using bump functions to estimate search directions may be a better option in such cases. The convergence result for locally Lipschitz continuous functions is established as follows:

\begin{theorem}
   \citep{evans} Suppose that $f$ is locally Lipschitz in some open set $U \subset \mathbb{R}^N$. Let $D^{\alpha_i} u$ denote the first-order derivative with respect to $x_i$, then 
    \begin{itemize}
        \item $ f * \eta_\sigma \rightarrow f$ almost everywhere;

        \item $\| D^{\alpha_i} f - \frac{\partial (f * \eta_{\sigma})}{\partial x_i} \|_{L^2(U)} \rightarrow 0$;
    \end{itemize}
    as $\sigma \rightarrow 0$.
\end{theorem}

\section{Proofs Omitted in Section~\ref{sec:appro}} \label{app:appro}

The following lemma will be helpful:
\begin{lemma}  \label{lebesgue}
    (Lebesgue Differentiation Theorem, \cite{stein})
     Let $f: \mathbb{R}^N \rightarrow \mathbb{R}$ is locally integrable, then for a.e. $x_0 \in \mathbb{R}^N$, it has
     \begin{itemize}
         \item $\frac{1}{Vol(\mathcal{B}(x_0, r))} \int_{\mathcal{B}(x_0, r)} f \ dx \rightarrow f(x_0)$ as $r \rightarrow 0$;

         \item $\frac{1}{Vol(\mathcal{B}(x_0, r))} \int_{\mathcal{B}(x_0, r)} |f(x) - f(x_0)| \ dx \rightarrow 0$ as $r \rightarrow 0$;
     \end{itemize}
     where $\mathcal{B}(x_0, r)$ is the open ball of radius $r$ centered at $x_0$ and $Vol(\mathcal{B}(x_0, r))$ is the volume of the ball.
\end{lemma}

Specifically, a point $x_0$ at which $\frac{1}{Vol(\mathcal{B}(x_0, r))} \int_{\mathcal{B}(x_0, r)} |f(x) - f(x_0)| \ dx \rightarrow 0$ as $r \rightarrow 0$ is called a \textit{Lebesgue point} of $f$.

\paragraph{Proof of Theorem~\ref{th:convergence}:}

    (a) Suppose that $\theta_0 \in \mathbb{R}^N$ is a Lebesgue point of $f$. Let $|f(\theta)| \leq M$ for all $\theta \in \mathbb{R}^N$, then for any $\epsilon > 0$, there exists $K > 0$ such that
    \begin{equation}   \label{proof1}
        \int_{|\theta - \theta_0| \geq r} | f(\theta) - f(\theta_0) | \ p_\sigma (\theta_0 - \theta) \ d\theta \leq 2 M \int_{|\theta - \theta_0| \geq r}  \ p_\sigma (\theta_0 - \theta) \ d\theta < \frac{\epsilon}{2}
    \end{equation}
    for all $r \geq K \sigma$ and $\sigma > 0$, which is possible because $p_\sigma$ is a Gaussian density function. According to Lemma~\ref{lebesgue}, there exists $r_0 > 0$ such that for any $r < r_0$, 
    $$\frac{1}{Vol(\mathcal{B}(\theta_0, r))} \int_{|\theta - \theta_0| < r} | f(\theta) - f(\theta_0) | \ d\theta < \frac{\epsilon}{2} \frac{(2 \pi)^{\frac{N}{2}}}{C (K + 1)^N }$$
    where $C = \frac{Vol(\mathcal{B}(x_0, r))}{r^N}$ depends only on the dimension $N$.

    On the other hand, the integral over $\mathcal{B}(x_0, r)$ is bounded by
    \begin{align*}
        \int_{|\theta - \theta_0| < r} | f(\theta) - f(\theta_0) | \ p_\sigma (\theta_0 - \theta) \ d\theta &= \frac{1}{(2 \pi)^{\frac{N}{2}} \sigma^N} \int_{|\theta - \theta_0| < r} | f(\theta) - f(\theta_0) | \ e^{-\frac{|\theta - \theta_0|^2}{2 \sigma^{2N}}} \ d\theta  \\
        &\leq \frac{1}{(2 \pi)^{\frac{N}{2}} \sigma^N} \int_{|\theta - \theta_0| < r} | f(\theta) - f(\theta_0) | \ d\theta  \\
        &\leq \frac{1}{(2 \pi)^{\frac{N}{2}} \sigma^N} \frac{\epsilon}{2} \frac{(2 \pi)^{\frac{N}{2}}}{(K + 1)^N} r^N  \\
        &\leq \frac{\epsilon}{2}  (\frac{r}{(K + 1) \sigma})^N 
    \end{align*}
    In particular, plugging $r = (K + 1) \sigma$ into the last inequality yields
    \begin{equation}   \label{proof2}
        \int_{|\theta - \theta_0| < r} | f(\theta) - f(\theta_0) | \ p_\sigma (\theta_0 - \theta) \ d\theta < \frac{\epsilon}{2}
    \end{equation}
    when $\sigma < \frac{r_0}{K + 1}$. 
    Combining \eqref{proof1} and \eqref{proof2} yields
    $$\int_{\mathbb{R}^N} | f(\theta) - f(\theta_0) | \ p_\sigma (\theta_0 - \theta) \ d\theta < \epsilon.$$

    Since $|f * p_\sigma (\theta_0) - f(\theta_0) | \leq \int_{\mathbb{R}^N} | f(\theta) - f(\theta_0) | \ p_\sigma (\theta_0 - \theta) \ d\theta$,
    using the fact that Lebesgue points are almost everywhere in $\mathbb{R}^N$ completes the proof.

    (b) According to the Rademacher's Theorem, $f$ is differentiable almost everywhere in $U$ where $U \subset \mathbb{R}^N$ is an open set. For each $x_i$, applying the definition of weak derivative yields
    \begin{equation}   \label{proof3}
        D^{\alpha_i} (f * p_\sigma) (\theta)  = (D^{\alpha_i} f ) * p_\sigma (\theta)
    \end{equation}
    for a.e. $\theta \in U$ since weak derivatives coincide with strong derivatives when the function is smooth. Specifically, the right-hand side of \eqref{proof3} exists because $f$ is globally Lipschitz continuous so that $D^{\alpha_i} f$ is bounded almost everywhere (and hence $D^{\alpha_i} f \in L^{\infty}(\mathbb{R}^N)$), which implies that the integral $\int_{\mathbb{R}^N} (D^{\alpha_i} f ) (y - x) \ p_\sigma (x) \ dy$ exists.
    
    Applying the result of part (a) completes the proof.

\newpage
\section{Details of experiments}
\label{app:detail}

The hyperparameters used in Section~\ref{sec:experiment} is summarized in Table~\ref{sample-table}. Note that all gradients are normalized.

\begin{table}
\begin{center}
\begin{small}
\begin{tabular}{lcccr}
\toprule
\quad & Quadrotor & Double pendulum & Hopper \\
\midrule
Batch    & 16 & 16 & 32 \\
Epoch    & 50 & 50 & 100 \\
Horizon    & 1000 & 1000 & 1000 \\
Discount factor    & 0.99 & 0.99 & 0.99 \\

\bottomrule

\end{tabular}
\end{small}
\end{center}
\caption{Hyperparameters for experiments.}
\label{sample-table}
\end{table}

\subsection{Hopper stand}

Regarding the controller, we use a 2-layer neural network $u = W_2 \tanh{W_1 s}$ where the width of the hidden layer is 16. the activation function is \emph{tanh}. The reward function $R(s, a) = -(s^2_2 + s^2_3 + s^2_4 + s^2_5 + 0.1 \ (s^2_6 + s^2_7 + s^2_8 + s^2_9 + s^2_{10} + s^2_{11})) - 0.001 \ |a|^2$ where the coordinates are specified as in Table~\ref{sample-table2}. The stepsize for policy update in each epoch is $\delta = 1$.

\begin{table}
\begin{center}
\begin{small}
\begin{tabular}{lcccr}
\toprule
Coordinate & Observation  \\
\midrule
$s_2$    & Angle of the torso  \\
$s_3$    & Angle of the thigh joint  \\
$s_4$    & Angle of the leg joint  \\
$s_5$    & Angle of the foot joint  \\
$s_6$    & Horizontal velocity the torso  \\
$s_7$    & Vertical velocity the torso  \\
$s_8$    & Angular velocity of the angle of the torso  \\
$s_9$    & Angular velocity of the angle of the thigh hinge  \\
$s_{10}$    & Angular velocity of the angle of the leg hinge  \\
$s_{11}$    & Angular velocity of the angle of the foot hinge  \\

\bottomrule

\end{tabular}
\end{small}
\end{center}
\caption{What the coordinates correspond to.}
\label{sample-table2}
\end{table}

\subsection{Double pendulum}

The dynamics of double pendulum (Figure~\ref{fig:plots} (b)) is solved from the manipulation equation
$$M(q) \ \ddot{q} + C(q, \dot{q}) \ \dot{q} = \tau(q) + B u$$
where 
$$M(q) = \begin{bmatrix}
I_1 + I_2 + m_2 l^2_1 + 2 m_2 l_1 l_{c2} \cos(\theta_2) & I_2 + m_2 l_1 l_{c2} \cos(\theta_2) \\
I_2 + m_2 l_1 l_{c2} \cos(\theta_2) & I_2 \\
\end{bmatrix}, \quad q = \begin{bmatrix}
\theta_1 \\
\theta_2 \\
\end{bmatrix}, \quad B = \begin{bmatrix}
1 & 0\\
0 & 1 \\
\end{bmatrix},$$
$$C(q, \dot{q}) = \begin{bmatrix}
-2 m_2 l_1 l_{c2} \sin(\theta_2) \dot{\theta}_2 & -m_2 l_1 l_{c2} \sin(\theta_2) \dot{\theta}_2  \\
m_2 l_1 l_{c2} \sin(\theta_2) \dot{\theta}_1  & 0 \\
\end{bmatrix}, $$
$$ 
\tau(q) = \begin{bmatrix} 
-m_1 g l_{c1} \sin(\theta_1) - m_2 g (l_1 \sin(\theta_1) + l_{c2} \sin(\theta_1 + \theta_2)) \\
-m_2 g l_{c2} \sin(\theta_1 + \theta_2) \\
\end{bmatrix}, \quad 
u = \begin{bmatrix}
u_1 \\
u_2 \\
\end{bmatrix}
$$
with $I_1 = I_2 = 0.1, m = 0.15, g = 9.81, l_1 = l_2 = 0.5, l_{c1} = l_{c2} = 0.25$. The controller $u = -Ks$ is linear and the initial feedback gain matrix is 
$$K_0 = \begin{bmatrix}
-20 & -20.0854 & -21.4826 & -10.0516 \\
-18.22 & -19.143 & -9.2905 & -6.6695 \\
\end{bmatrix}.$$

The reward function is $R(s, a) = 5  (\theta^2_1 + \theta^2_2) + 0.5 (\dot{\theta}^2_1 +\dot{\theta}^2_2) + 0.00005  |a|^2$. The stepsizes for simulation and policy updates are $\Delta t = 0.01$ and $\delta = 1$. The control inputs are saturated, i.e., $u = [\max\{-10, \min\{u_1, 10\} \}, \max\{-10, \min\{u_2, 10\} \}]^T$.

\begin{figure}[h!]
\centering

\subfigure[]{\includegraphics[width=0.45\textwidth]{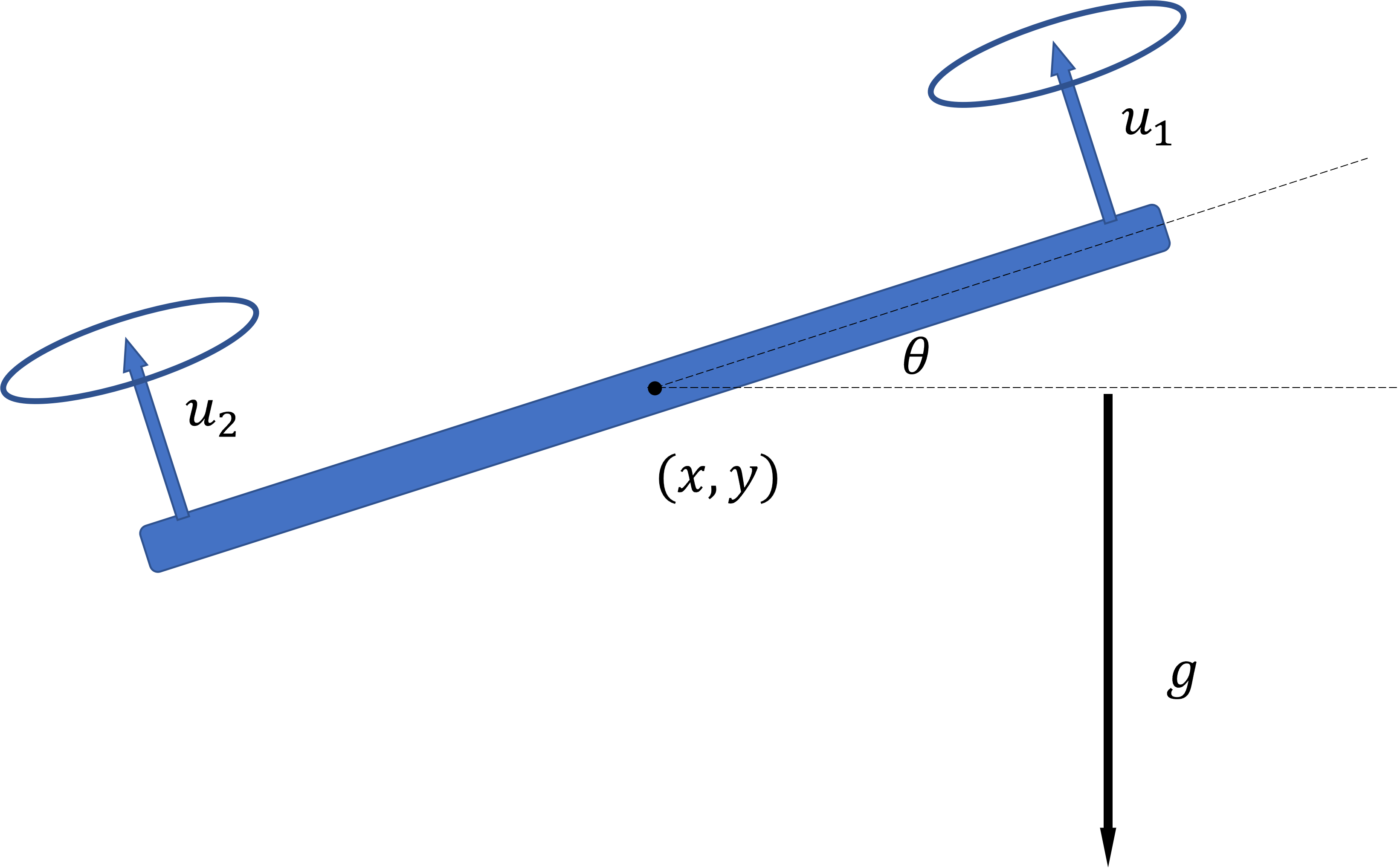}}
\subfigure[]{\includegraphics[width=0.45\textwidth]{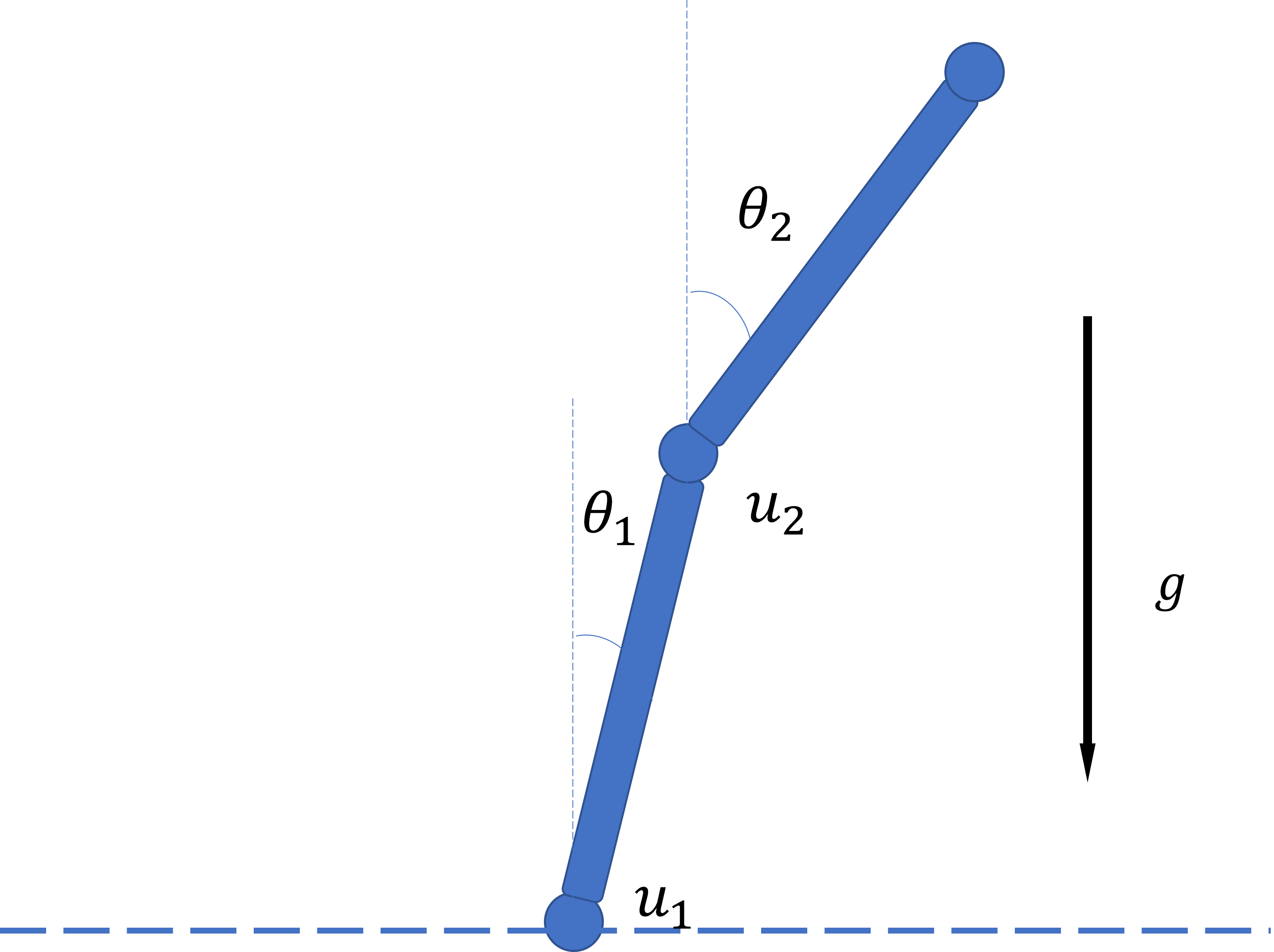}}

\caption{(a) Planar quadrotor; (b) double pendulum.}
\label{fig:plots}

\end{figure}

\subsection{Planar quadrotor}

The dynamics of the planar quadrotor (Figure~\ref{fig:plots} (a)) is 
\begin{align*}
    \dot{x}_1 &= x_2 \\
    \dot{x}_2 &= -\frac{1}{m} (u_1 + u_2) \sin{\theta} \\
    \dot{y}_1 &= x_2 \\
    \dot{y}_2 &= \frac{1}{m} (u_1 + u_2) \cos{\theta} - mg \\
    \dot{\theta} &= w \\
    \dot{w} &= \frac{r}{I} (u_1 - u_2) 
\end{align*}
where $m = 1.0, I = 0.1, r = 0.5, g = 9.81$ and $u_1, u_2 \in \mathbb{R}$ are the control inputs. The reward function is $R(s, a) = - (x^2_1 + y^2_1 + \theta^2 + 0.1 \ (x^2_2 + y^2_2 + w^2) + 0.0001 \ |a|^2)$ and $\Delta t = 0.1$ is the stepsize of discretization. The controller $u = -Ks + b$ is linear and the initial policy is $u_0 = -K_0 s + b_0$ with
$$K_0 = \begin{bmatrix}
-2.2361 & -3.3404 & 2.2361 & 2.69 & 13.5092 & 2.7752 \\
2.2361 & 3.3404 & 2.2361 & 2.69 & -13.5092 & -2.7752 \\
\end{bmatrix}, \quad 
b_0 = \begin{bmatrix}
4.905  \\
4.905 \\
\end{bmatrix}.
$$

$u_0$ is the optimal policy solved by LQR methods. The standard deviation of Gaussian policy is fixed to $\sigma = 0.1$ and the stepsize for policy updates in each epoch is $\delta = 0.001$. The control inputs are saturated, i.e., $u = [\max\{0, \min\{u_1, 10\} \}, \max\{0, \min\{u_2, 10\} \}]^T$

\newpage
\section{Behaviors of Hopper}
\label{sec:behavior}

The behaviors of the hopper with different variances are presented below.

\begin{figure}[h!]
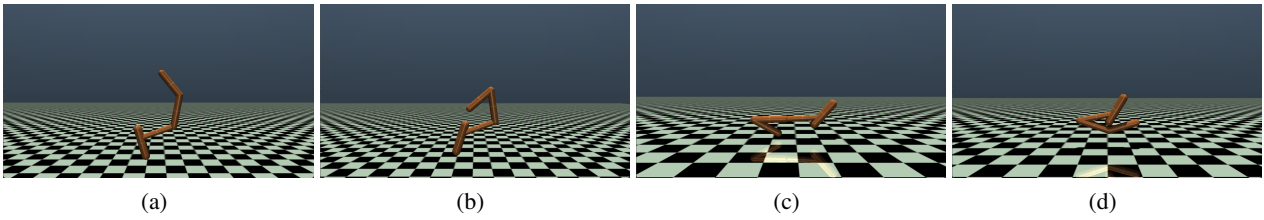

\centering

\subfigure[]{\includegraphics[width=0.24\textwidth]{0.0051.png}}
\subfigure[]{\includegraphics[width=0.24\textwidth]{0.0052.png}}
\subfigure[]{\includegraphics[width=0.24\textwidth]{0.0053.png}}
\subfigure[]{\includegraphics[width=0.24\textwidth]{0.0054.png}}

\caption{$\sigma = 0.005$.}

\end{figure}

\begin{figure}[h!]
\centering

\subfigure[]{\includegraphics[width=0.24\textwidth]{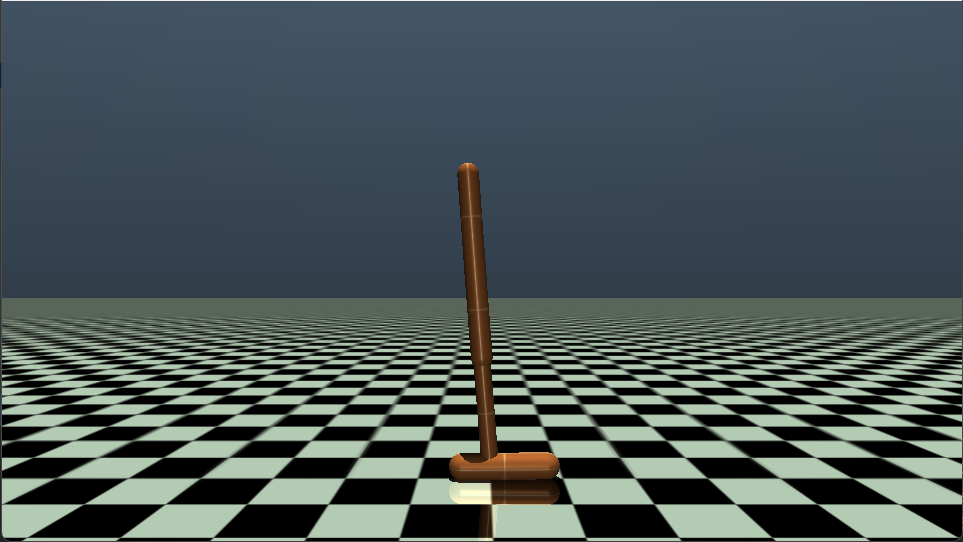}}
\subfigure[]{\includegraphics[width=0.24\textwidth]{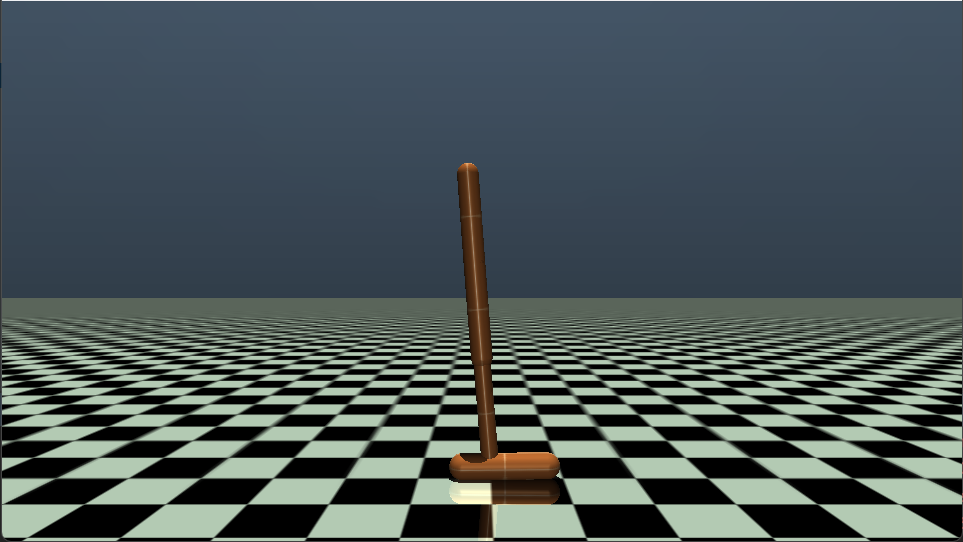}}
\subfigure[]{\includegraphics[width=0.24\textwidth]{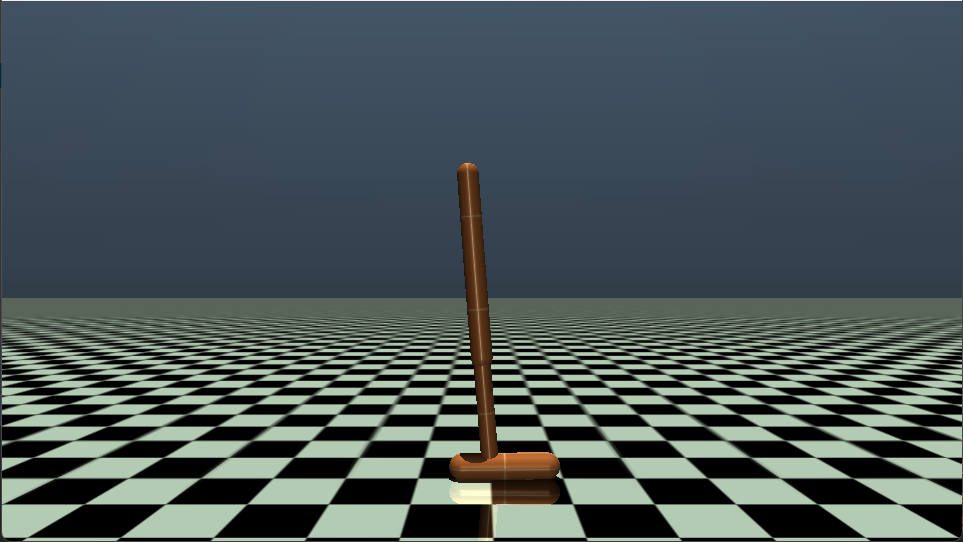}}
\subfigure[]{\includegraphics[width=0.24\textwidth]{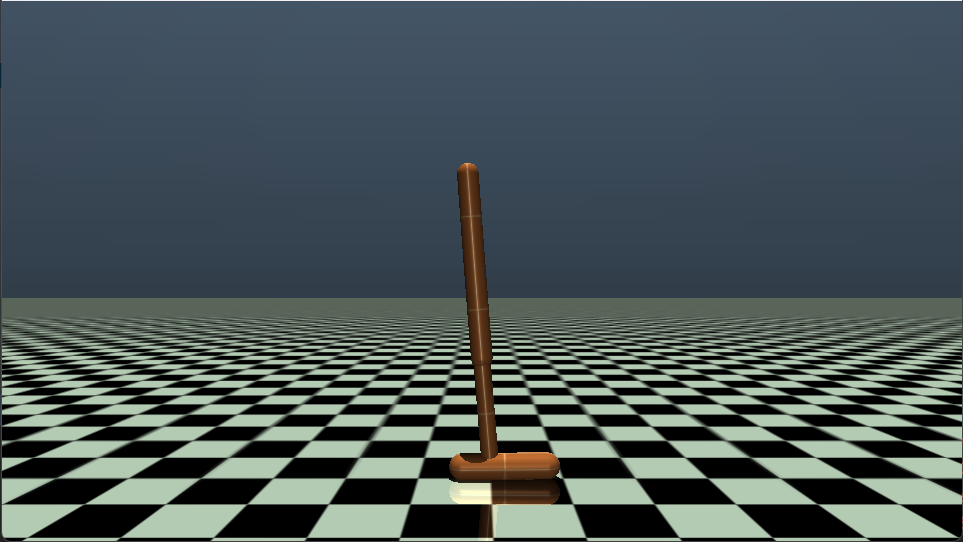}}

\caption{$\sigma = 0.05$.}

\end{figure}

\begin{figure}[h!]
\centering

\subfigure[]{\includegraphics[width=0.24\textwidth]{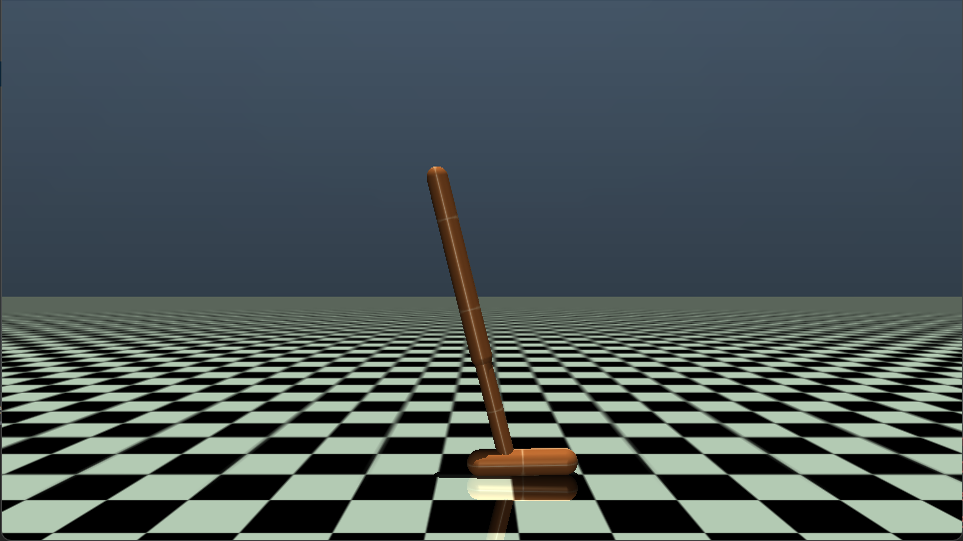}}
\subfigure[]{\includegraphics[width=0.24\textwidth]{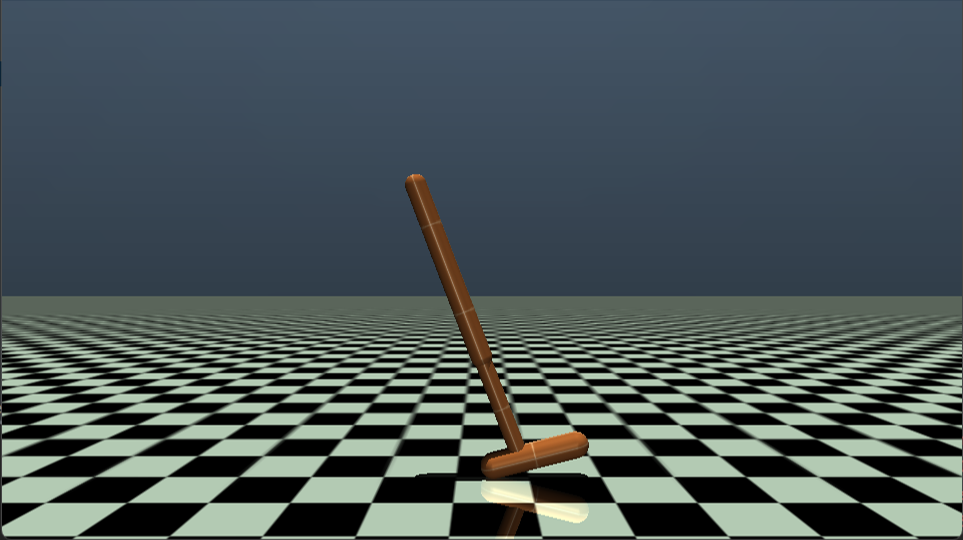}}
\subfigure[]{\includegraphics[width=0.24\textwidth]{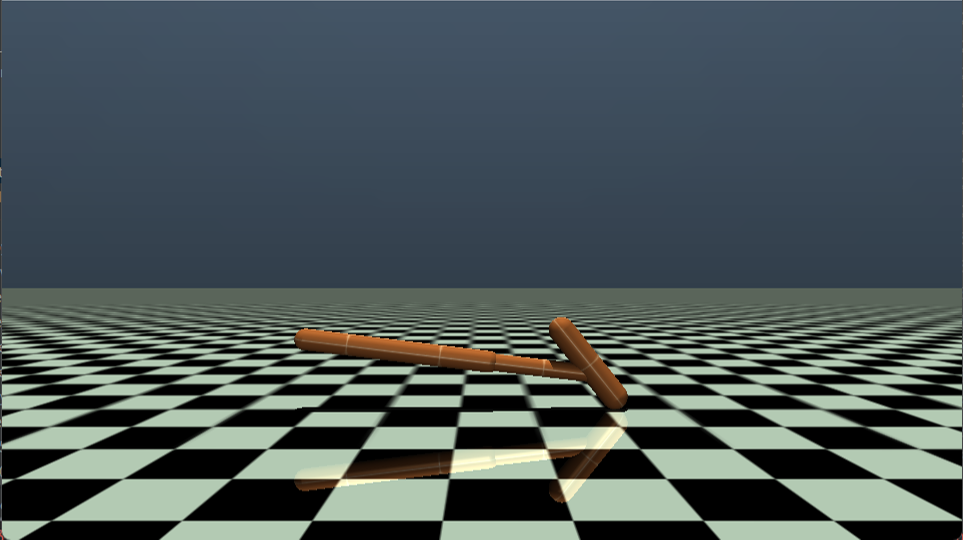}}
\subfigure[]{\includegraphics[width=0.24\textwidth]{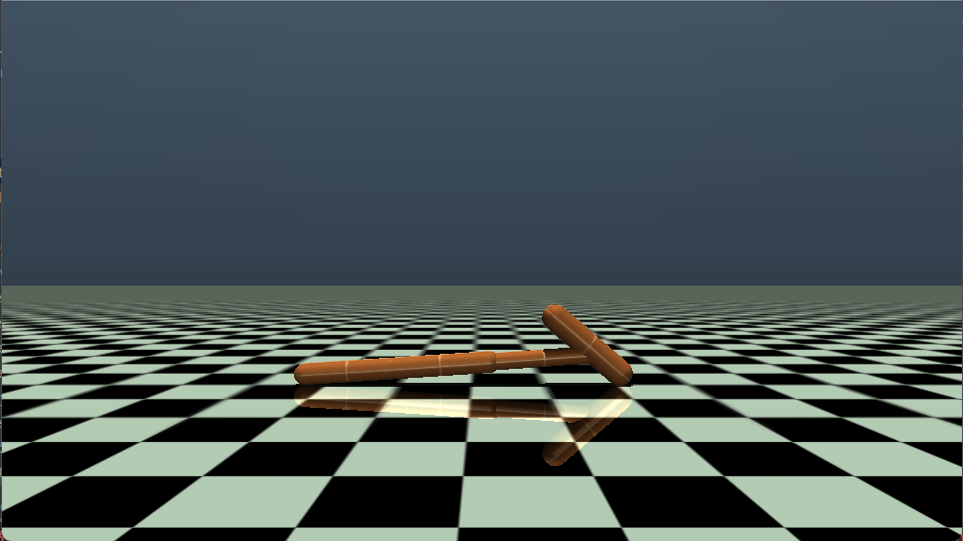}}

\caption{$\sigma = 0.5$.}

\end{figure}

\begin{figure}[h!]
\centering

\subfigure[]{\includegraphics[width=0.24\textwidth]{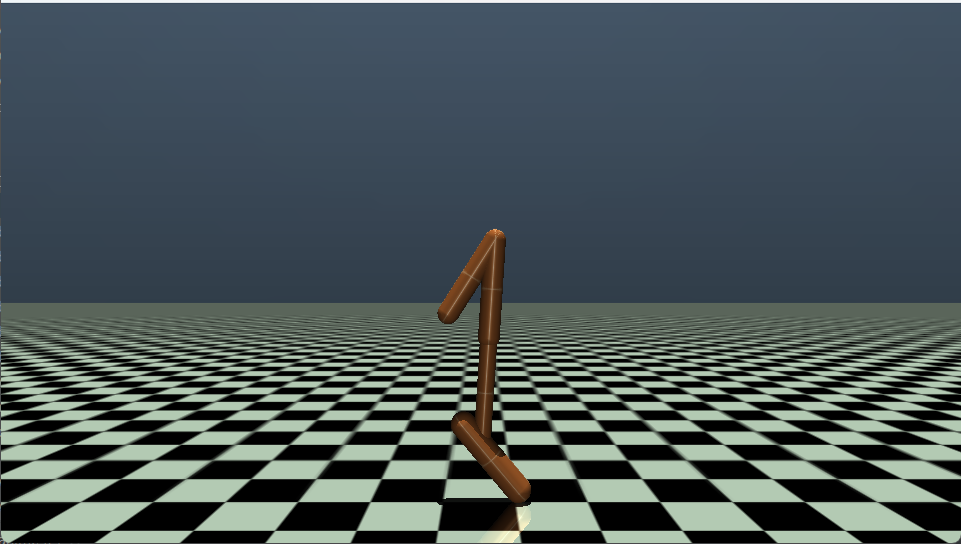}}
\subfigure[]{\includegraphics[width=0.24\textwidth]{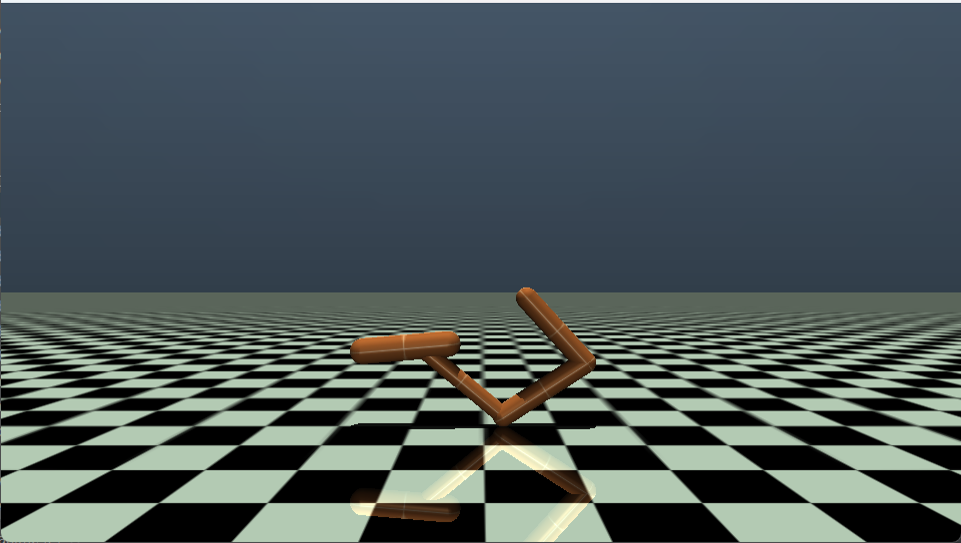}}
\subfigure[]{\includegraphics[width=0.24\textwidth]{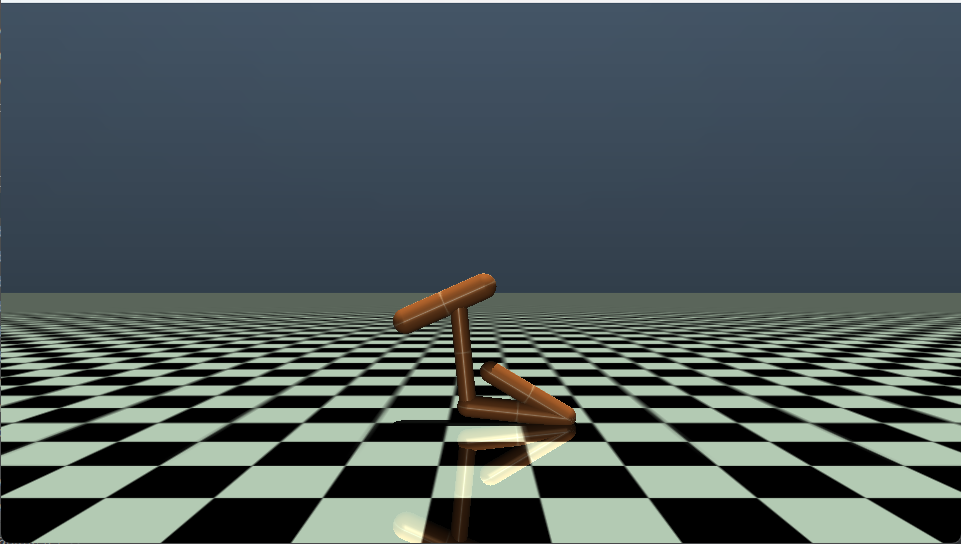}}
\subfigure[]{\includegraphics[width=0.24\textwidth]{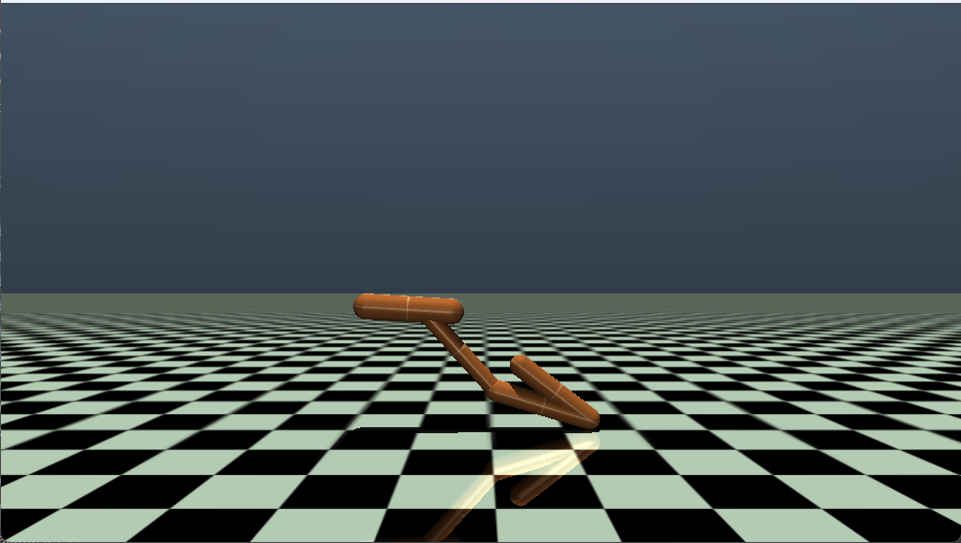}}

\caption{$\sigma = 5$.}

\end{figure}


\end{document}